\documentclass{article}

\usepackage[verbose=true,letterpaper]{geometry}
  \newgeometry{
    textheight=9in,
    textwidth=5.5in,
    top=1in,
    headheight=12pt,
    headsep=25pt,
    footskip=30pt
  }

\usepackage{parskip}

\PassOptionsToPackage{numbers}{natbib}
\usepackage{natbib}
\PassOptionsToPackage{dvipsnames}{xcolor}

\usepackage{hyperref}

\usepackage{float}
\usepackage[utf8]{inputenc} 
\usepackage[T1]{fontenc}    
\usepackage{url}            
\usepackage{booktabs}       
\usepackage{amsfonts}       
\usepackage{nicefrac}       
\usepackage{microtype}      
\usepackage{subcaption}
\usepackage{amsmath}
\usepackage{tabularx}
\usepackage{mathabx}
\usepackage{amssymb}
\usepackage{enumitem}
\usepackage{cleveref}
\usepackage{wrapfig}
\usepackage{multirow}
\usepackage{tcolorbox}
\usepackage{adjustbox}
\usepackage{makecell}

\definecolor{our-blue}{HTML}{1f77b4}

\usepackage[utf8]{inputenc} 
\usepackage[T1]{fontenc}    
\usepackage{hyperref}       
\hypersetup{
    colorlinks=true,
    citecolor=our-blue,
    linkcolor=our-blue,
    filecolor=magenta,      
    urlcolor=our-blue,
}

\definecolor{lightblue}{HTML}{84C7F9}
\definecolor{lighterblue}{HTML}{D4ECFF}
\newtcolorbox{mybox}{colback=lighterblue,colframe=lightblue}

\usepackage[toc,page,header]{appendix}
\usepackage{minitoc}

\date{}

\usepackage{amsthm}
\newtheorem{theorem}{Theorem}[section]
\newtheorem{proposition}[theorem]{Proposition}

\title{Bayesian Optimization of Antibodies Informed by a Generative Model of Evolving Sequences}

\author{
\normalsize \textbf{Alan Nawzad Amin}$^1$ 
\quad \textbf{Nate Gruver}$^{*1}$ \quad \textbf{Yilun Kuang}$^{*1}$ \quad \textbf{Lily Li}$^{*1}$ \\ 
\normalsize \textbf{Hunter Elliott}$^{2}$ \quad \textbf{Calvin McCarter}$^{2}$ \quad \textbf{Aniruddh Raghu}$^{2}$\\
\normalsize \textbf{Peyton Greenside}$^{2}$ \quad \textbf{Andrew Gordon Wilson}$^{1}$ \\
\normalsize $^1$New York University \quad $^2$BigHat Biosciences
}

\begin{document}

\doparttoc
\faketableofcontents 

\maketitle

\begin{abstract}
\let\thefootnote\relax\footnotetext{* Equal contribution}
    To build effective therapeutics, biologists iteratively mutate antibody sequences to improve binding and stability. Proposed mutations
    can be informed by previous measurements or by learning from large antibody databases to predict only typical antibodies. Unfortunately, the space of typical antibodies is enormous to search, and experiments often fail to find suitable antibodies on a budget. We introduce Clone-informed Bayesian Optimization (CloneBO), a Bayesian optimization procedure that efficiently optimizes antibodies in the lab by teaching a generative model how our immune system optimizes antibodies. Our immune system makes antibodies by iteratively evolving specific portions of their sequences to bind their target strongly and stably, resulting in a set of related, evolving sequences known as a \emph{clonal family}. We train a large language model, CloneLM, on hundreds of thousands of clonal families and use it to design sequences with mutations that are most likely to optimize an antibody within the human immune system. We propose to guide our designs to fit previous measurements with a twisted sequential Monte Carlo procedure. We show that CloneBO optimizes antibodies substantially 
    more efficiently than previous methods in realistic \textit{in silico} experiments and designs stronger and more stable binders in \textit{in vitro} wet lab experiments.
\end{abstract}

\section{Introduction}

Antibody therapeutics are the fastest growing class of drugs, with approved treatments for a breadth of disorders ranging from cancer to autoimmune disease to infectious disease \citep{carter2018next}. Biologists wish to design antibodies that strongly bind to targets of interest while being stable in the human body. Stable antibodies do not unfold or cause an adverse immune reaction~\citep{jarasch2015developability}.
To develop these antibodies, biologists first screen many diverse antibody sequences, or use a lab animal's immune system to find an initial candidate that binds a target.
This candidate often does not bind strongly or is unstable in the human body, so it is used as a starting point in an iterative optimization experiment in which biologists predict mutations that result in better or more stable binders \citep{Lu2020-qy}.

To make these predictions, we can learn from up to thousands of sequence measurements from many previous iterations \citep{Rapp2024-mo, Yang2019-xe, Fannjiang2022-jr, Brookes2019-fv}.
We can also learn from databases of protein sequences to avoid predicting mutations that produce nonfunctional antibodies \citep{Gruver2023-sf, Stanton2022-ru, Hie2023-sr, Prihoda2022-zy}.
However, even with this restriction, there are a combinatorial number of mutations we could predict, only a handful of which are beneficial.
Therefore, optimization experiments regularly fail to find suitable sequences on a budget.

To optimize more efficiently than current methods, we need an informed prior about where and how to mutate to positively affect binding and stability.
Ideally we could learn what mutations often lead to better sequences in optimization experiments in the lab.
Unfortunately such data is scarce.
In principle, we can instead learn from abundant data about what mutations often lead to better sequences in our bodies.
To make an antibody that binds a new target, our immune system evolves sets of related sequences known as \emph{clonal families};
through selection, sequences in clonal families accumulate mutations that increase binding to a target while maintaining stability \citep{Burnett2018-fo}.
Through large-scale sequencing efforts, we can now learn from databases that contain large numbers of these evolving sequences \citep{Olsen2022-nt}.
\begin{figure}
    \centering
    \includegraphics[width=0.95\textwidth]{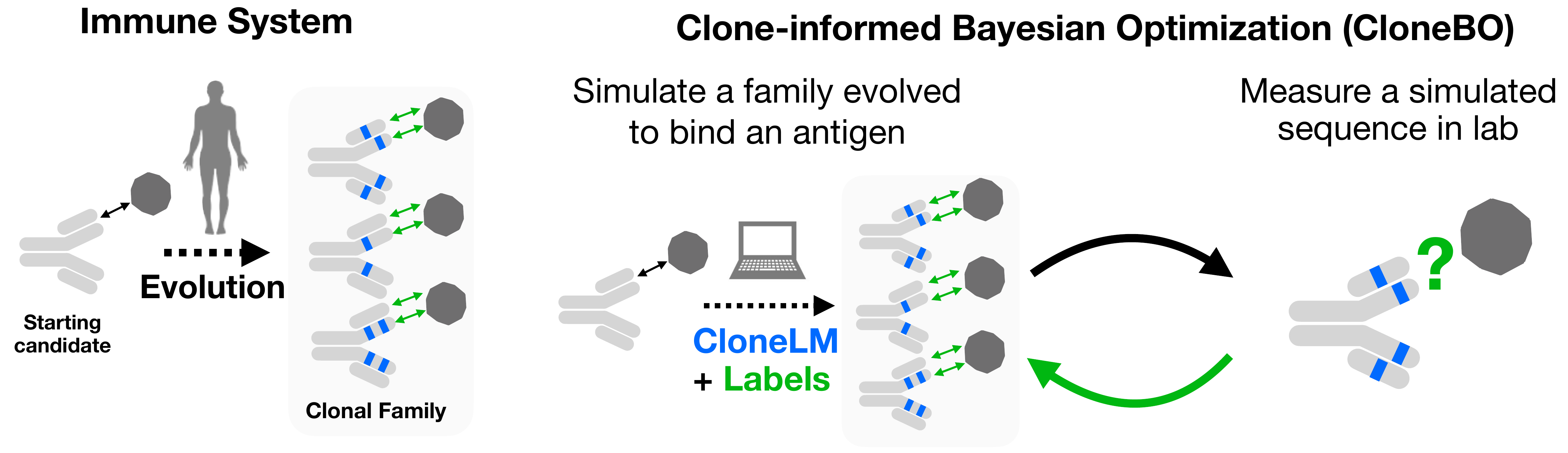}
    \caption{Our immune system introduces mutations (blue) to evolve weak binders of a target into strong binders (green). 
    The result is a set of related sequences that bind the antigen strongly and stably known as a clonal family. We use a model trained on these families, CloneLM, to perform Bayesian optimization in a procedure called CloneBO. We use experimental data to generate a clonal family that might have evolved to bind our antigen and suggest sequences to test in the lab.
    }\label{fig:concept}
\end{figure}

In this paper, we introduce Clone-informed Bayesian optimization (CloneBO), a Bayesian optimization procedure which efficiently optimizes antibody sequences in the lab by teaching a generative model how the human immune system optimizes antibodies (Fig.~\ref{fig:concept}).
In Section \ref{sec: related work} we review related work. In Section \ref{sec: background} we introduce the problem of iterative Bayesian optimization.
In Section \ref{sec: de finetti} we describe how in theory we can build a prior for where and how to mutate given observed clonal families.
In Section \ref{sec: clonelm} we build such a prior in practice by fitting a large language model, CloneLM, to hundreds of thousands of clonal families.
We take a martingale posterior approach to sampling in which we generate new clonal families that contain our candidate.
In Section \ref{sec: clonebo} we describe how to condition on previous measurements using a twisted sequential Monte Carlo procedure so that good mutations are included in our clonal family, and bad mutations are excluded.
We use our model to build a Bayesian optimization procedure, CloneBO.
In Section \ref{sec: results} we show that CloneBO optimizes realistic oracles for stability and binding strength much more efficiently than current methods and also designs strong and stable binders in wet lab experiments.
CloneBO outperforms naive and informed greedy methods as well as LaMBO, a state of the art method for optimizing sequences. 
In Section \ref{sec: conclusion} we conclude and describe directions for future work.

Our code and model weights are available at 
\url{https://github.com/AlanNawzadAmin/CloneBO}.

\section{Related work}\label{sec: related work}

To iteratively optimize a protein, one can predict sequences using previous measurements \citep{Rapp2024-mo, Yang2019-xe, Fannjiang2022-jr, Brookes2019-fv}.
To optimize antibodies for stability in the human body, \citet{Hie2023-sr} and \citet{Prihoda2022-zy} suggest introducing mutations to select for typicality, make them look more typical, measured by the likelihood of a model trained on large databases of protein sequences.
More generally, \citet{Gruver2023-sf} and \citet{Stanton2022-ru} avoid suggesting atypical protein sequences by training a latent space to represent a database of sequences and then optimizing in this latent space.
However, even the space of typical antibodies is combinatorially large, and therefore challenging to search using only up to thousands of previous measurements.
CloneBO builds an informed prior to efficiently search this space.

CloneBO builds this prior using clonal families --- sets of sequences evolving to strongly and stably bind a target \citep{Burnett2018-fo}.
Biologists infer evolutionary pressures on antibodies by examining individual clonal families \citep{Mascola2013-dr} or comparing clonal families \citep{Phad2022-bh}.
In the lab, ``repertoire mining'' optimizes antibodies by suggesting mutations from sequences in a clonal family that contains the candidate \citep{Richardson2021-xi, Olsen2023-fk}.
In practice, such a family rarely exists. CloneBO optimizes a candidate by generating new clonal families that contain the candidate and that match experimental data.

To build a prior over measurements in the lab, we assume that sequences in a clonal family are distributed with abundance according to their fitness, and that fitness is close to the function we measure in the lab.
These are standard assumptions in generative modelling of protein sequences \citep{Weinstein2022-tv} --- one fits a distribution $p$ to a set of protein sequences, then uses $\log p(X)$ as an estimate of the fitness of a sequence $X$; this fitness then correlates strongly with the function of the protein as measured in the lab \citep{Frazer2021-hn, Riesselman2018-nm, Notin2022-qw, Shin2021-lc}.
In our case, each clonal family has its own fitness function which we use to build a prior over fitness functions.
Our model, CloneLM, models clonal families as sets of sequences, similar to the architecture of \citet{Truong2023-dk} who used a language model to model protein families as sets of sequences.

For each clonal family we observe a set of sequences which evolve with respect to a latent fitness function drawn from a prior.
Instead of attempting to build an explicit latent variable model, \citet{Fong2024-bx} suggest performing Bayesian inference with a ``martingale posterior''. Instead of sampling and conditioning on the latent variable, they do the same with a large number of observations.
\citet{Lee2022-ij} suggests using this approach for Bayesian optimization.
\citet{Falck2024-fm} suggests that, with some bias, large language models can perform martingale posterior inference.
We take this approach when sampling from our prior. 
We use a large language model to flexibly fit observed sequences and sample sets of sequences, i.e. clonal families, as draws from our prior.

When proposing sequences, we sample a clonal family from an autoregressive language model, CloneLM, but condition its output to fit experimental measurements.
To do so, we build a twisted sequential Monte Carlo procedure \citep{Whiteley2014-db} in which we bias the generation of each letter towards the posterior.
This technique is used to sample from filtering models \citep{Lawson2023-hc, Lawson2024-jd}, large language models \citep{Zhao2024-au}, or diffusion models \citep{Trippe2023-ej}.

Complementary with work on iterative design are structure-based \textit{de novo} design methods which aim to predict antibody sequences that bind a particular antigen \citep{Jin2021-ub, Luo2022-ha, Kong2023-yu}.
These models have the potential to design starting sequences for iterative optimization.
These models could in principle also be used for iterative design, but cannot make use of a pool of previous measurements, and must have access to structure.
We show below empirically that these models are not well suited for this task.

\section{Background}\label{sec: background}

We start with an antibody variable domain $X_0$, a sequence of $110\sim130$ letters made of the 20 amino acids, identified to bind a target of interest. 
$X_0$ often does not bind the target strongly enough
or is unstable in the human body, making it unsuitable as a therapeutic. 
We therefore iteratively propose sequences we expect are stronger or more stable binders, $\hat X_1, \dots, \hat X_N$, and measure their binding or stability in the lab $Y_1, \dots, Y_N$.

We assume that our measurements are evaluations of a function $f$ that takes sequences to a scalar measurement of binding or stability in the lab: $Y_n=f(\hat X_n)$.
To suggest the next sequence, $\hat X_{N+1}$, given $\hat X_{1:N}, Y_{1:N}$ we can perform Bayesian optimization \cite{Frazier2018-ze}.
First we place a prior on $f$ given our known weak or unstable binder $X_0$, $p(f|X_0)$.
Then we infer $f$ by building a posterior, $p(f|X_0, \hat X_{1:N}, Y_{1:N}).$
Finally, we suggest $\hat X_{N+1}$ given our knowledge of $f$, for example by Thompson sampling: we sample a value of $f$ we believe to be plausible, $f\sim p(f|X_0, \hat X_{1:N}, Y_{1:N})$, and test the sequence that maximizes this sample, $\hat X_{N+1}=\mathrm{argmax}_Xf(X)$.

In the lab we often have a limited experimental budget, and therefore want to find a strong or stable binder in as few iterations as possible.
To do so, we need an accurate prior on $f$.
The ideal prior could in principle be constructed by performing many optimization experiments in the lab for a diverse array of targets and starting candidates and measuring $f$ in each case.
Unfortunately, performing a large number of these experiments is prohibitively expensive.

\section{A prior from fitness functions of evolving antibody sequences}\label{sec: de finetti}

While we do not have access to a large number of optimization experiments in the lab, we do have access to a large number of similar optimization experiments that occur in our bodies.
Our immune system generates antibodies by first identifying candidate sequences that bind a target.
It then evolves this sequence towards binding its target more strongly while remaining stable in the body:
    mutations are introduced to sequences and those sequences with higher ``fitness'' --- those that bind the target more strongly and stably --- are selected for reproduction.
Each starting candidate sequence typically produces many diverse sequences that have been evolved to bind a target strongly and stably.
For each optimization experiment the immune system performs, we therefore observe a set of evolved sequences $X_1, \dots, X_M$ known as a ``clonal family''.

The $f$ we measure in the lab measures binding and stability, similar to a function of the fitness of sequences under selection.
Therefore to build a prior over $f$ we start by building a prior over fitness functions $F$.
Then in Section~\ref{sec: clonebo} we allow for some discrepancy between $f$ and $F$ which may be caused by a difference between measurements in the lab and selection in our bodies.

To get a prior over functions $F$ from observed clonal families, we first note that the distribution of sequences we observe, $p(X_1, X_2, \dots)$, can be written as a Bayesian model.
The probability of observing a set of sequences in a clonal family is \textit{exchangable}, i.e. it does not depend on their order; so, by De Finetti's theorem\footnote{We ignore the dependence between the number of sequences we observe, $M$, and the sequences themselves} \citep{Hewitt1955-ua}, sequences in each clonal family are generated iid conditional on a latent random variable which we call $\mathrm{clone}$:
\[p(X_{1:M}) = \int \prod_{m=1}^M p(X_m|\mathrm{clone})p(\mathrm{clone}).\]
Next we make the standard assumption that families of evolving proteins are distributed with abundance proportional to their fitness \citep{Weinstein2022-tv}, that is, 
\begin{equation}\label{eqn: clone to fitness}
    F(X)=\log p(X|\mathrm{clone}).
\end{equation}
\citet{Sella2005-to} showed that Eqn.~\ref{eqn: clone to fitness} holds exactly if a protein evolves under $F$ over long time scales. 
In reality, sequences drawn from $p(X|\mathrm{clone})$ can also be correlated by being descendants of the same sequence, but we make the standard assumption that these correlations can be ignored \citep{Weinstein2022-tv}.
Finally, we can represent that the initial candidate $X_0$ binds the target by assuming we have observed it in the clonal family, i.e., by looking at $p(\mathrm{clone}|X_0)\propto p(\mathrm{clone})p(X_0|\mathrm{clone})$.

We can therefore sample fitness functions from $p(F|X_0)$ in theory by \textbf{1)} sampling clonal families that contain $X_0$, $\mathrm{clone}^*\sim p(\mathrm{clone}|X_0)$, and then \textbf{2)} we can set $F(X)=\log p(X|\mathrm{clone}^*)$.

\section{CloneLM: Learning a prior over fitness functions}\label{sec: clonelm}
In this section we fit a model to the distribution of clonal families and use it to sample fitness functions in practice.
In principle, we could build a model with an explicit latent variable meant to represent $\mathrm{clone}$.
Instead, we take a martingale posterior approach~\cite{Fong2024-bx} --- simply by building an accurate model of clonal sequences we learn an implicit prior on $\mathrm{clone}$ that we can approximately sample from.

In Section~\ref{sec: train clonelm} we fit an autoregressive large language model, CloneLM, to large scale clonal family data and show it can generate realistic clonal families $X_{1:M}$ that contain a candidate sequence $X_0$.
In Section~\ref{sec: martingale} we show that given a clone, $X_{0:M}$, CloneLM implicitly infers the fitness function when predicting sequences: $F(X)\approx \log p_{\mathrm{CloneLM}}(X_{M+1}|X_{0:M})$.
Finally in Section~\ref{sec: fitting llm} we show CloneLM can therefore sample fitness functions from an implicit prior on $\mathrm{clone}$ by generating realistic clonal families that contain $X_0$ then inferring their fitness functions.

\subsection{Fitting a large language model to generate clonal families}\label{sec: train clonelm}

\begin{figure}[ht!]
    \centering
    \includegraphics[width=\textwidth]{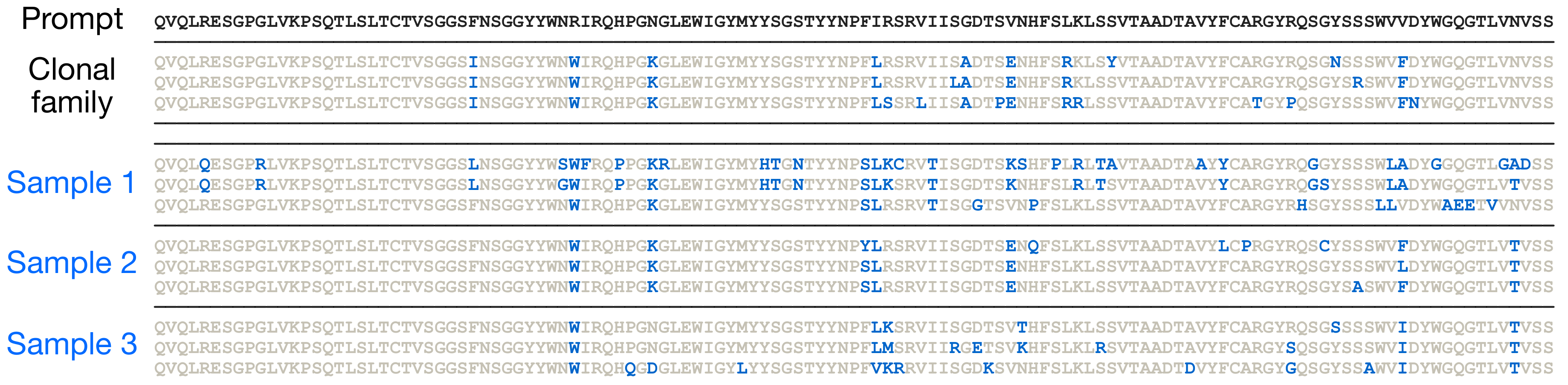}
    \caption{\textbf{CloneLM samples plausible clones.} We compare sequences in a clonal family to families generated by CloneLM conditional on $X_0$ ("Prompt").
    We align sequences to $X_0$ and highlight locations where sequences differ from $X_0$ in blue. The sampled clonal families 
    have variants in similar places, are similarly diverse as the real one, 
    and share similar variants within each family.
    }
    \label{fig: clone ex}
\end{figure}

We train a large language model on large scale human clonal family data.
Each antibody is made up of two amino acid sequences --- the ``light chain'' and the ``heavy chain''.
We train separate models on heavy and light chains of antibody sequences in clonal families.

To build a training set, we collect all sets of human heavy and light chain antibody sequences from the database of Observed Antibody Space (OAS) \citep{Olsen2022-nt}.
We annotate clonal families in each set of sequences using FastBCR \citep{Wang2023-vj} and remove any clonal family with fewer than 25 sequences.
Our dataset contains 908 thousand heavy chain clonal families and 34 thousand light chain clonal families.

We then train autoregressive language models with 377 million parameters based on the Mistral-7B architecture on the heavy and light chain datasets \citep{Jiang2023-yy}.
We represent each clonal family as a sequence of tokens made up of all the amino acid sequences in the clonal family each separated by a special sequence-separator token.
We place spaces between amino acids so that the tokenizer represents each amino acid with its own token.
Our model, CloneLM, accurately fits this data --- it achieves a test perplexity of 1.276 on the heavy chain data and 1.267 on the light chain data.
We provide details of the data curation and of training the models in App.~\ref{app: training data}.

To see if CloneLM generates realistic clonal families, in Fig.~\ref{fig: clone ex} we compare a heavy chain clonal family from the test set to clonal families generated by CloneLM conditional on a randomly selected sequence from the original family $X_0$.
We see the sampled clonal families are similarly diverse to the real clonal family, include variants in similar locations as the true clonal family, and sequences within the same sampled clonal families contain similar variants.
We show more examples of generated heavy and light chain clonal families in Appendix~\ref{app: clone examples}.

\begin{figure}[ht!]
    \centering
    \begin{subfigure}[b]{0.182\textwidth}
        \includegraphics[width=\textwidth]{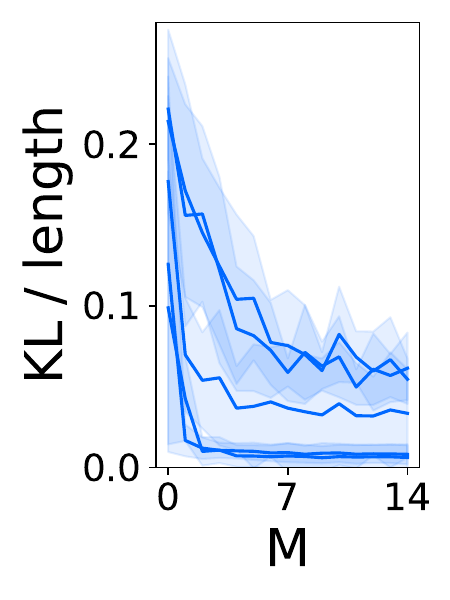}
        \vspace{-6mm}
        \captionsetup{labelformat=empty}
        \caption{\hspace{0.7cm}(a)}
        \label{fig: kl dec}
    \end{subfigure}
    \begin{subfigure}[b]{0.182\textwidth}
        \includegraphics[width=\textwidth]{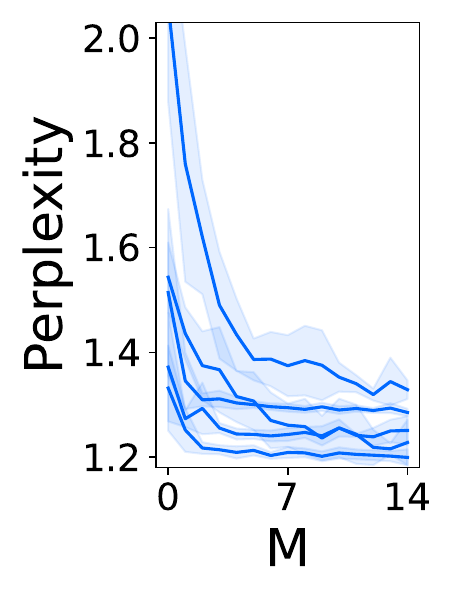}
        \vspace{-6mm}
        \captionsetup{labelformat=empty}
        \caption{\hspace{0.8cm}(b)}
        \label{fig: nll dec}
    \end{subfigure}
     \hspace{5mm}
    \begin{subfigure}[b]{0.3\textwidth}
        \raisebox{0.16\height}{\includegraphics[width=\textwidth]{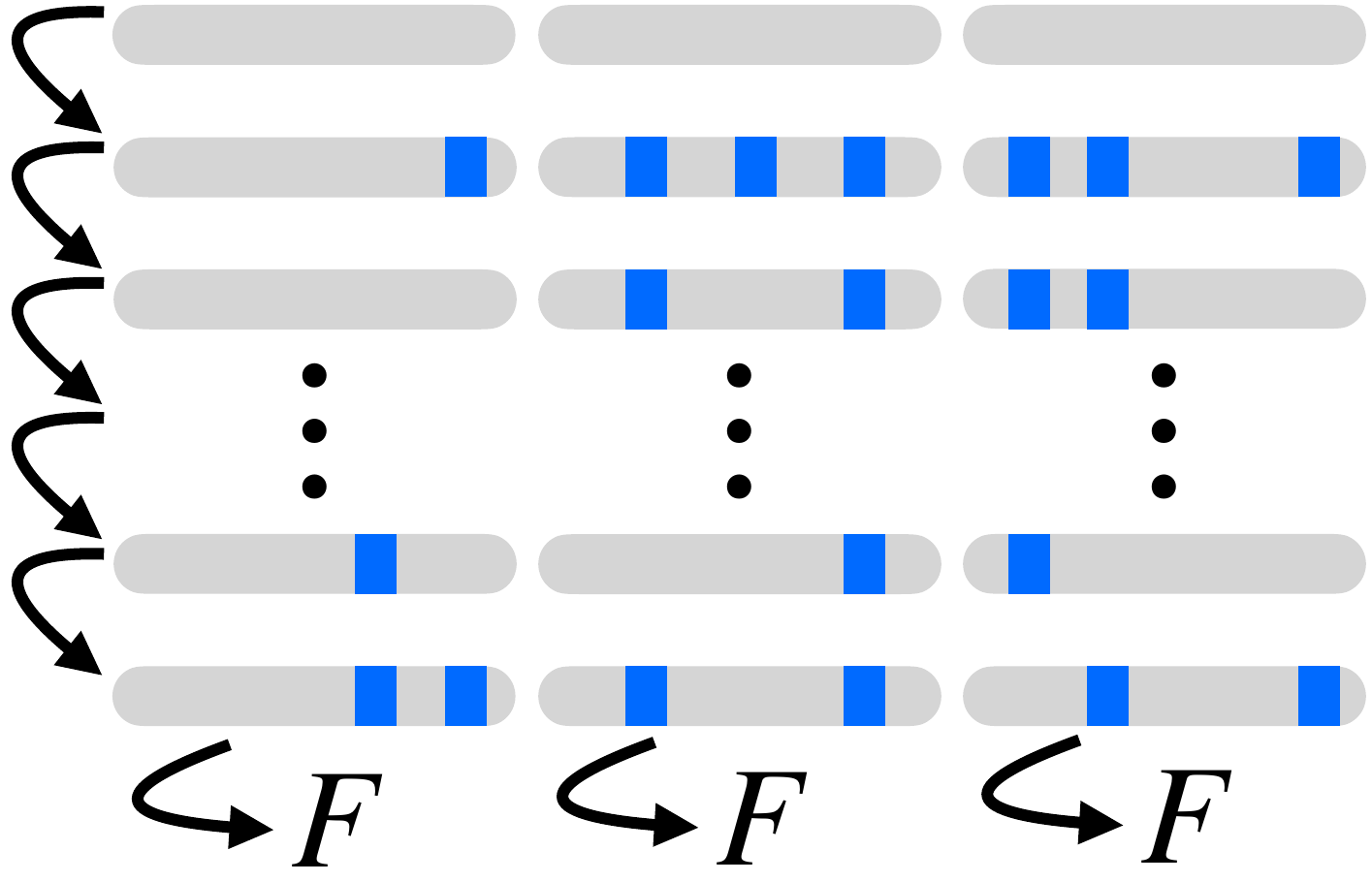}}
        \vspace{-6mm}
        \caption{}
        \label{fig: concept fitness}
    \end{subfigure}
    \hspace{5mm}
    \begin{subfigure}[b]{0.18\textwidth}
        \includegraphics[width=\textwidth]{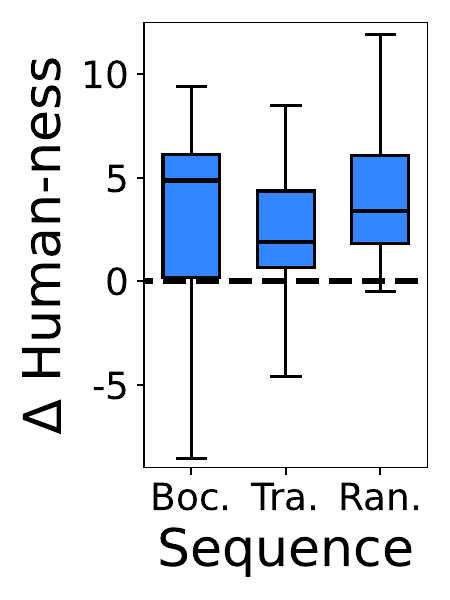}
        \vspace{-6mm}
        \captionsetup{labelformat=empty}
        \caption{\hspace{0.8cm}(d)}
        \label{fig: humanization}
    \end{subfigure}
    \caption{\textbf{CloneLM is a prior over fitness functions.} (a) For 5 different clonal families, with starting sequences $X_0$, $p_{\mathrm{CloneLM}}(X_{M+1}|X_{1:M}')$ gets close to $p_{\mathrm{CloneLM}}(X_{M_{\mathrm{large}}+1}|X_{1:M_\mathrm{large}})$ in KL.
    We shade one standard deviation across 10 samples of $X_{1:M_{\mathrm{large}}}, X'_{1:M}$.
    (b) For 5 different heavy chain clonal families, $p_\mathrm{CloneLM}(X|X_{0:M})$ better predicts sequences in a clonal family when conditioned on more sequences from that same clonal family $X_{0:M}$.
    We shade one standard deviation across 10 samples of $X_{1:M}$.
    (c) To sample from our prior $F\mid X_0$ we perform a martingale posterior procedure.
    (d) We evolve three antibody therapeutics with three mutations from 25 sampled fitness functions.
    These sequences evolve to look more like human antibodies.
    }
    \label{fig:fitness}
\end{figure}

\subsection{Approximately extracting a fitness landscape from a clonal family}\label{sec: martingale}

CloneLM does not explicitly represent the latent variable $\mathrm{clone}$ so we cannot exactly query the fitness function $F(X)=\log p(X|\mathrm{clone})$.
However, CloneLM approximates the predictive distribution of a Bayesian model $p(X_{M+1}|X_{0:M})$ which implicitly integrates over the latent $\mathrm{clone}$:
\[p(X_{M+1}|X_{0:M})=\int p(X_{M+1}|\mathrm{clone})dp(\mathrm{clone}|X_{0:M}).\]
As $M\to\infty,$ $p(\mathrm{clone}|X_{0:M})$ converges to a point mass at the latent $\mathrm{clone}^*$ generating the sequences.
Therefore, in theory, as $M$ becomes large, $\log p_{\mathrm{CloneLM}}(X_{M+1}|X_{0:M})$ should converge to $F$:
\[\log p_{\mathrm{CloneLM}}(X_{M+1}|X_{0:M})\approx\log p(X_{M+1}|X_{0:M})\approx \log p(X_{M+1}|\mathrm{clone}^*)=F(X_{M+1}).\]

We see that CloneLM can infer $F$ as such on real data as well --- as $M$ becomes large, the predictive of CloneLM, $p_{\mathrm{CloneLM}} (X_{M+1}|X_{1:M})$, approaches convergence and its limit increasingly approaches the distribution of sequences in a clonal family, $p(X|\mathrm{clone})$.
First in Fig.~\ref{fig: kl dec} we take random sequences $X_{0:M_{\mathrm{large}}}, X'_{0:M}$ from heavy chain clonal families and see if $p_\mathrm{CloneLM}(X_{M+1}|X_{0:M}')$ converges to $p_\mathrm{CloneLM}(X_{M_{\mathrm{large}}+1}|X_{0:M_{\mathrm{large}}})$ in Kullback-Leibler divergence. 
Setting $M_{\mathrm{large}}=24$, we see that
    although the divergence does not go to $0$, the distributions become very similar as $M$ becomes large.
Next in Fig.~\ref{fig: nll dec} we take random sequences $X_{0:M}$ from heavy chain clonal families and see if $p_\mathrm{CloneLM}(X|X_{0:M})$ approaches $p(X|\mathrm{clone})$.
We see indeed that when we use $p_\mathrm{CloneLM}(X|X_{0:M})$ to predict sequences in the clonal family, its perplexity is decreasing in $M$.

\subsection{Sampled fitness landscapes and evolving sequences}\label{sec: fitting llm}

Given that our model samples realistic clonal families and it can recover $F(X)=\log p(X|\mathrm{clone})$, we can approximately sample from $p_\mathrm{CloneLM}(F|X_0)$ using a martingale posterior procedure~\cite{Fong2024-bx} --- we sample from the prior of clonal families that contain $X_0$, $X_{1:M}\sim p_{\mathrm{CloneLM}}(X_{1:M})$, and then approximate the fitness function as $F(X)\approx \log p_{\mathrm{CloneLM}}(X|X_{0:M})$ (Fig.~\ref{fig: concept fitness}).

The fitness functions we sample reflect how our immune systems evolve antibodies.
In Fig.~\ref{fig: humanization} we take the heavy chains of three antibody therapeutics, bococizumab, trastuzumab, and ranibizimab, sample fitness functions setting from CloneLM conditional on these sequences with $M=10$, and iteratively evolve these sequences by adding the most likely mutation under each sampled fitness functions three times.
These therapeutics are not originally human sequences and therefore can be unstable in the human body --- in particular, bococizumab was discontinued due to harsh side effects.
If they were to evolve in our bodies, we would expect them to become more human-like, and therefore likely more stable.
Indeed we see that as we evolve these sequences they look more like human antibodies, where human-ness is measured by the likelihood of IgLM \cite{Shuai2022-gp}, a model trained on a large set of human antibodies.
\section{CloneBO: Inference with experimental measurements}\label{sec: clonebo}

We now describe how to use our prior over fitness functions $F$ to optimize sequences in the lab.
In Section~\ref{sec: lab prior} we build a prior for measurements in the lab $f$ using our prior for $F$.
We cannot exactly condition on the implicit $\mathrm{clone}$, so in Section~\ref{sec: approx post} we approximate the posterior over the latent $\mathrm{clone}$ with a posterior over concrete clonal families $X_{1:M}$.
In Section~\ref{sec: tsmc} we describe how to sample from our approximate posterior using a twisted sequential Monte Carlo (SMC) procedure.
Finally in Section~\ref{sec: thompson} we describe how to suggest sequences to perform iterative optimization with CloneLM; we call our method Clone-informed Bayesian Optimization (CloneBO).

\subsection{Building a prior on laboratory measurements}\label{sec: lab prior}
In our bodies, antibodies are optimized to stably and strongly bind their targets.
In principle, we are interested in doing same in the lab.
We therefore assume that the function we optimize in the lab, $f$, is approximately drawn from our prior on fitness functions, $F$, while allowing for some discrepancy due to mismatches between measurements in the lab and those in our bodies. For simplicity, we assume that $f$ is an affine linear transformation of $F$ and that the deviation between experiment and fitness is independent normal with error $\sigma^2$; that is, for some $T>0, C$, calling $F_n = \log p(\hat X_{n}|\mathrm{clone})$,
\[
    Y_n \mid \mathrm{clone}\sim N(T F_n + C, \sigma^2I).
\]

To reflect our vague uncertainty about $T$ and $C$ we use uniform priors: $C\sim \mathrm{Uniform}(-\infty, \infty)$ and $T\sim \mathrm{Uniform}(0, \infty)$.
With these priors, we get an analytical expression for the marginal likelihood.

\begin{proposition}\label{prop: marginal lik}
    (Proof in App.~\ref{app: marginal lik proof}.) For some constant $D$, and $R=\sqrt{N}\frac{\mathrm{Std}(Y_{1:N})}{\sigma}\mathrm{Cor}(F_{1:N}, Y_{1:N})$, with $\Phi$ as the Gaussian CDF,
    \begin{equation}\label{eqn: likelihood}
        \log p(Y_{1:N}|F_{1:N}) = -\frac 1 2 \log \mathrm{Cov}(F_{1:N})  + \frac 1 2 R^2 + \log\Phi(R)+D.
    \end{equation}
\end{proposition}

The first term in Eqn.~\ref{eqn: likelihood} pushes the fitness values of the measured sequences, $F_{1:N}$, to be different, while the later terms push $F_{1:N}$ to strongly and positively correlate with $Y_{1:N}$.

As before, we assume we have a starting candidate $X_0$ that belongs in the clonal family.
Therefore,
\[p(\mathrm{clone}|Y_{1:N}, \hat X_{1:N}, X_0)\propto p(\mathrm{clone})p(X_0|\mathrm{clone})p\left(Y_{1:N}|F_{1:N}\right)\]

\subsection{Approximating the posterior}\label{sec: approx post}
To infer $f$ given measurements $Y_{1:N}, \hat X_{1:N}$, we would like to sample from the posterior $F\sim p(F|Y_{1:N}, \hat X_{1:N}, X_0)$. As we only implicitly represent $\mathrm{clone}$, we approximate the posterior by swapping the latent $\mathrm{clone}$ for a concrete clonal family, $X_{1:M}$; then, as in Sec.~\ref{sec: martingale}, we can approximately query the fitness function $F(X)\approx \log p(X|X_{0:M})$.

To create an approximate posterior, we replace $F_{n}$ with $F^M_n=\log p(\hat X_n|X_{0:M})$:
\begin{equation}\label{eq: approx post}
\tilde p_M(X_{1:M}|Y_{1:N}, \hat X_{1:N}, X_0)\propto p(X_{1:M}|X_0)p\left(Y_{1:N}|F^M_{1:N}\right).
\end{equation}
As $M\to\infty$, $F_{1:N}^M\to F_{1:N}$, so, $\tilde p_M$ converges to the distribution one obtains by sampling $\mathrm{clone}$ from the posterior and sampling $X_m\sim p(X_m|\mathrm{clone})$ iid:
\begin{proposition}\label{prop: post converge}
    (Proof in App.~\ref{app: approx post proof}) As $M\to\infty$, $\tilde p_M$ converges to the true posterior of $X_{0:M}$.
\end{proposition}
In practice, we approximate $p(X_{1:M}|X_0)$ and $F_{1:N}^M$ with CloneLM. 

\subsection{Sampling from the approximate posterior with twisted SMC}\label{sec: tsmc}

\begin{figure}
    \centering
    \begin{subfigure}[b]{0.3\textwidth}
        \includegraphics[width=\textwidth]{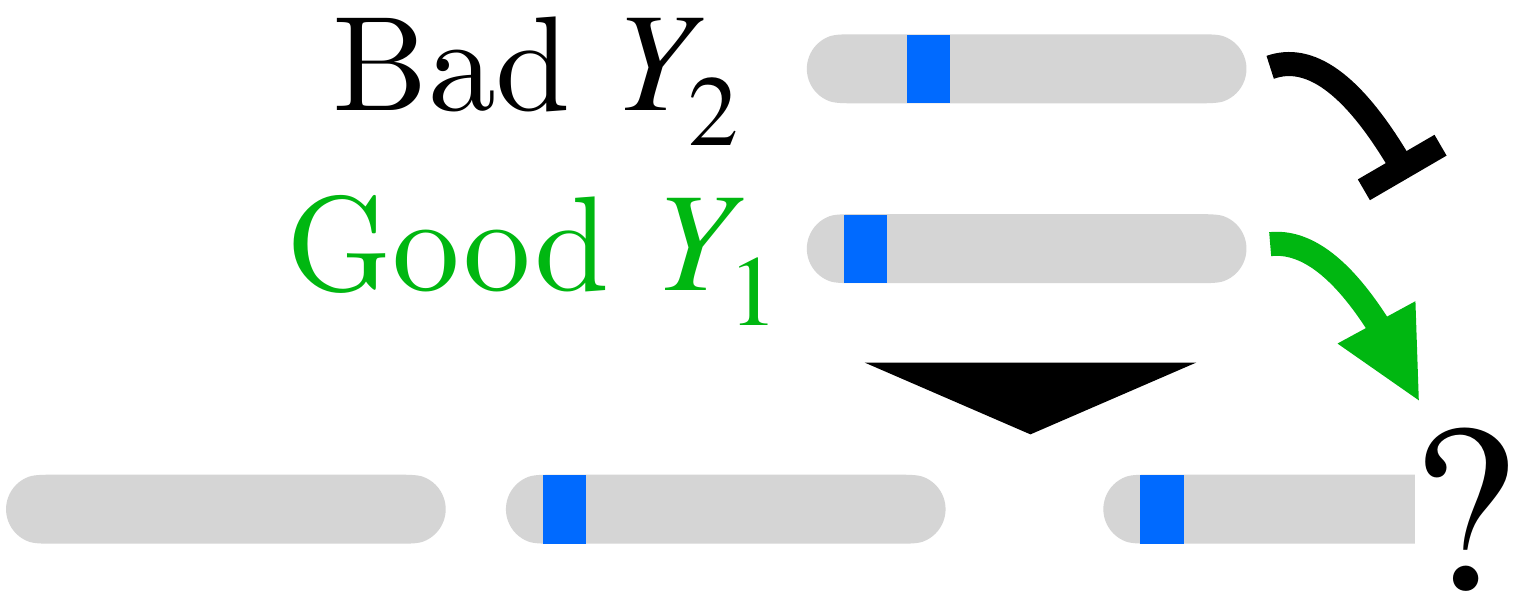}
        \vspace{0.4cm}
        \caption{}
        \label{fig: concept twisted}
    \end{subfigure}
    \begin{subfigure}[b]{0.365\textwidth}
        \includegraphics[width=\textwidth]{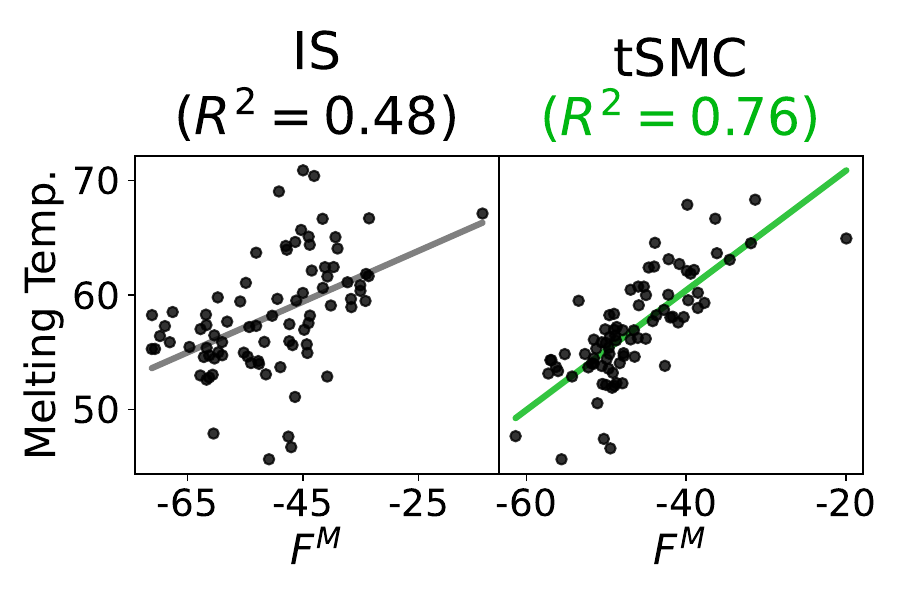}
        \caption{}
        \label{fig: tm fitness fit}
    \end{subfigure}
    \begin{subfigure}[b]{0.305\textwidth}
        \includegraphics[width=\textwidth]{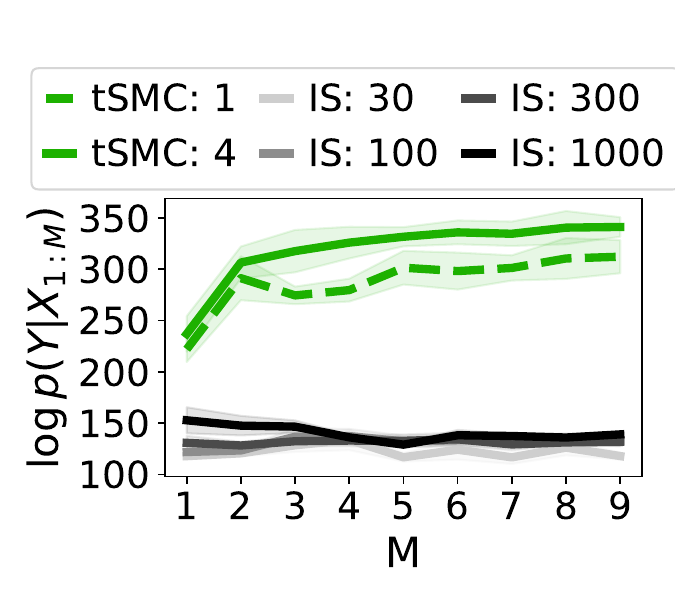}
        \captionsetup{labelformat=empty}
        \caption{\hspace{-0.8cm}(c)}
        \label{fig: tsmc}
    \end{subfigure}
    \caption{\textbf{We accurately sample functions from the posterior with a twisted SMC procedure.} a) To sample from our posterior, we bias our generated sequences to look more like those sequences that were measured in the lab to be good.
    b) A sample from tSMC better fits the data than an importance sample.
    We show a line of best fit between $F_{1:N}^M$ (fitness from a clone of $M$ sequences) and $Y_{1:N}$ (measurement) for example clonal families sampled by importance sampling or twisted SMC, with $M=6$, $D=300$ particles for IS, and $D=4$ for twisted SMC.
    c) We quantify the result from (b) across 10 replicates for various clone sizes $M$.
    }
    \label{fig:tsmc}
\end{figure}

To sample from $\tilde p_M$, we
need to generate $M$ sequences from CloneLM such that the probabilities of the $(M+1)$st sequence matches experimental measurements.
Naively, we might sample $X_{1:M}$ from CloneLM and then importance sample.
However, the space of sequences is large, and we may fail to resample a mutation that improved measurements in the lab.

Instead we bias generation at each letter by adding measured sequences to our clonal family so that mutations that improve measurements are encouraged and those that harm measurements are avoided (Fig.~\ref{fig: concept twisted}).
We call this bias a ``twisted'' distribution, a term from the sequential Monte Carlo literature. In practice, this bias does not exactly sample from the posterior, so we efficiently correct for the discrepancy with a sequential Monte Carlo procedure.

\paragraph{Twisted distributions}
Define $X^{(l)}$ to be the letter at the $l$th position of a sequence $X$ and $X^{(:l)}$ as the first $l$ letters in $X$; if $l$ is greater than the length of $X$, we define $X^{(l)}$ as an empty position.
To build our twisted distributions, we decompose the likelihood of $\hat X_n$ given $M+1$ sequences into contributions from the first $M$ sequences and contributions from each letter in sequence $M+1$:
\begin{equation}\label{eqn: f contrib}
    \begin{aligned}
        F^{M+1}_n=\log p(\hat X_n|X_{0:M+1}) = & \log \frac{p( X_{M+1}|X_{0:M}, \hat X_n)}{p( X_{M+1}|X_{0:M})}+\log p(\hat X_n|X_{0:M})\\
        = &\sum_{l}\log\frac{p( X_{M+1}^{(l)}|X_{0:M}, X_{M+1}^{(:l-1)}, \hat X_n)}{p( X_{M+1}^{(l)}|X_{0:M}, X_{M+1}^{(:l-1)})} +F^{M}_n\\
        =: &\sum_{l}F_n^{M+1, (l)} +F^{M}_n.
    \end{aligned}
\end{equation}
We can calculate $F_n^{M+1, (l)}$ by adding $\hat X_n$ to the end of the clonal family $X_{1:M}$ and calculating how much the conditional likelihood of the $l$-th letter of $X_{M+1}$ increases. To approximately sample from the posterior therefore, when sampling each letter we bias towards letters with $F_n^{M+1, (l)}$ that better match experimental measurements (Fig.~\ref{fig: concept twisted});
we call these approximations of the posterior the ``twisting distributions'' and we define them formally in App.~\ref{app:sequential_is}.

The twisted distributions are not the exact marginals.
We correct for this discrepancy by sampling $D>1$ sequences at a time and importance sampling in a sequential Monte Carlo procedure (SMC), which we also describe in App.~\ref{app:sequential_is}.
The final method is known as twisted SMC, and when $D=1$, it is equivalent to sampling from the approximations described above.

\paragraph{Empirical results}
We sample a clonal family conditional on laboratory measurements of the melting temperature of 75 related antibodies from an experiment described in Sec.~\ref{sec: results}.
In Fig.~\ref{fig: tm fitness fit} and \ref{fig: tsmc} we see that clonal families from twisted SMC fit the experimental data substantially better than clonal families importance sampled from unconditional samples from CloneLM.
We also see in Fig.~\ref{fig: tsmc} that correcting for bias with $D=4$ also improves the fit to the data and that as $M$ increases, the likelihood $p(Y_{1:N}|F_{1:N}^M)$ plateaus, reflecting the convergence of $\tilde p_M$ to the true posterior.
We show similar results for laboratory measurements of binding in App.~~\ref{app:tsmc affinity res}

\subsection{Bayesian optimization with CloneBO}\label{sec: thompson}

After sampling $F$ from the posterior we would like to suggest sequences to test in lab.
We take a Thompson sampling approach \citep{Russo2018-vr}: we propose testing the sequence predicted to maximize $F(X)$, and therefore maximize $f(X)$, in the lab.
We cannot optimize $F(X)$ over all sequences, so in practice we start with $4$ sequences with the highest measurements $Y$ and iteratively optimize $F(X)$ over the top substitution for up to $3$ substitutions.

In theory, $X_0$ represents a candidate sequence to optimize.
In practice, we found it helpful to take a greedy approach --- we randomly select $X_0$ from the $4$ sequences with the highest measurements $Y$.

Conditioning on a large number of measurements $\hat X_{1:N}, Y_{1:N}$ is computationally expensive.
To accommodate large $N$, other Bayesian optimization methods build summaries of the measurements, for example by fitting a neural network to them \cite{Stanton2022-ru}.
In our case, we only condition on the measurements of sequences predicted to be most informative -- we calculate the probability that each $\hat X_n$ appears in a clonal family with $X_0$, $p(X_0, \hat X_n)$ and condition on the measurements of the $75$ most likely sequences.
Additional details of CloneBO are provided in App.~\ref{app: clonebo details}.

\section{Experiments}\label{sec: results}

Now we demonstrate that CloneBO efficiently optimizes antibody sequences against oracles trained on real stability and affinity data.
In Section ~\ref{sec: in silico} we demonstrate that our method efficiently optimizes fitness functions or laboratory measurements \textit{in silico}.
In Section~\ref{sec: lab results} we also show that our method suggests mutations that optimize sequences in \textit{in vitro} experiments the lab.
We provide details of our experiments in App.~\ref{app: exp details}.

\subsection{Optimizing antibodies \textit{in silico}}\label{sec: in silico}

We show that CloneBO efficiently optimizes fitness functions or measurements in the lab \textit{in silico}.
To simulate fitness functions or lab measurements, we train oracles $f$ on real data.
We show that CloneBO outperforms naive and informed baselines.

\subsubsection{Baselines}

First we consider a naive \textbf{Greedy} baseline which suggests a random substitution of one of the top 4 sequences.
We also compare to an informed greedy baseline which randomly picks one of the top 4 sequences and introduces the mutation predicted to make the antibody look most like a typical antibody, where typicality in measured by the likelihood of a masked language model trained on antibody sequences, \textbf{Sapiens} \citep{Prihoda2022-zy};
    this is a popular strategy for making antibodies that are more stable in the body.
We also compare to \textbf{LaMBO} \citep{Stanton2022-ru}, a state-of-the-art Bayesian optimization method for sequences which builds a latent space using a pool of sequences and fits experimental measurements in this latent space; by conditioning on this latent space, it is less likely to suggest atypical sequences.
We also build a LaMBO model informed by the space of antibodies by pretraining its latent space using 100000 antibody sequences from the observed antibody space database~\citep{Olsen2022-nt}, \textbf{LaMBO-Ab}. 
We also compare to other popular strategies for iterative optimization of sequences that do not have antibody-based priors, \textbf{Genetic}, \textbf{AdaLead}~\citep{sinai2020adalead}, and \textbf{EvoBO}~\citep{sinai2020adalead}; \textbf{CMA-ES}~\citep{Hansen2001-dn}, \textbf{Dyna-PPO}~\citep{Angermueller2020-pa}, and \textbf{CbAS}~\citep{Brookes2019-fv}.
In total, we compare CloneBO to 10 baselines that represent state-of-the-art industry practice\footnote{Note that some of these methods are developed for different regimens, such as short sequences or large amounts of training data. 
Their performance here does not necessarily represent their performance for the problems they are optimized for.}.

\subsubsection{Results} 
\begin{figure}

    \centering
    \begin{subfigure}[b]{0.23\textwidth}
        \includegraphics[width=\textwidth]{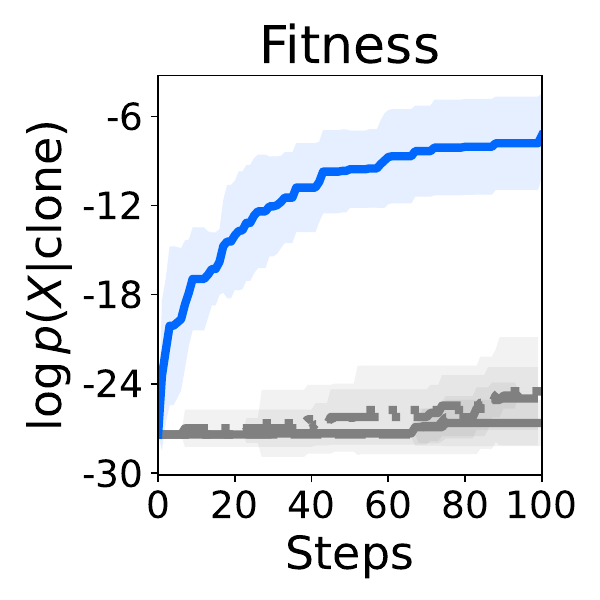}
        \captionsetup{labelformat=empty}
        \caption{\hspace{0.5cm}(a)}
        \label{fig: clone opt}
    \end{subfigure}
    \begin{subfigure}[b]{0.76\textwidth}
        \centering
        \includegraphics[width=\textwidth]{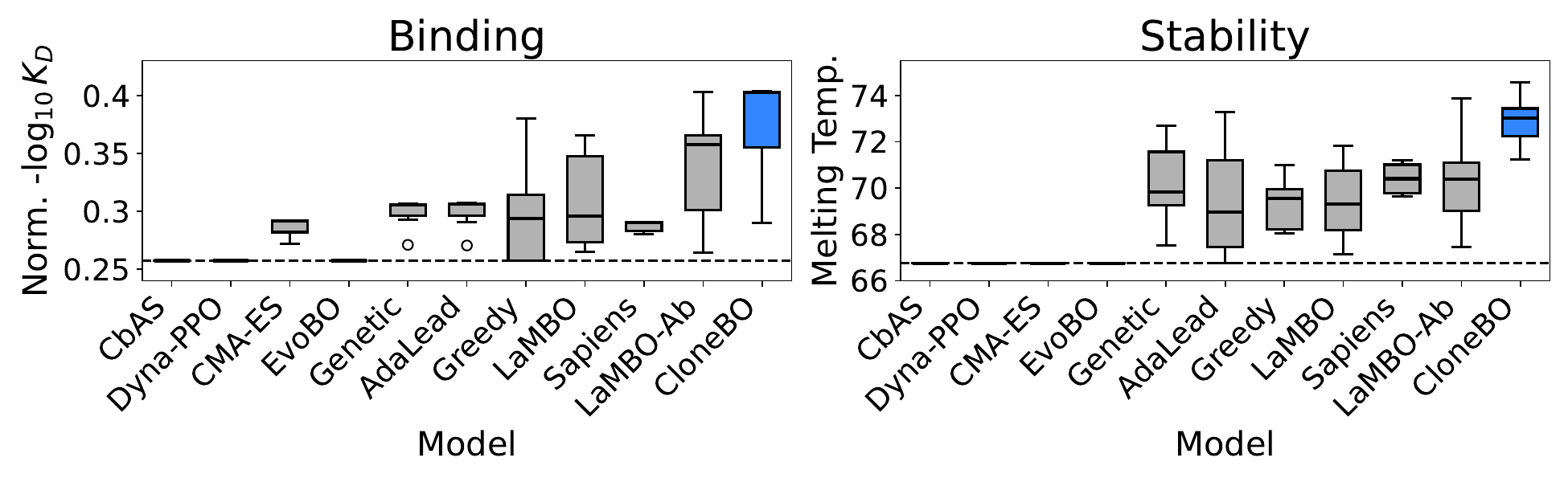}
        \captionsetup{labelformat=empty}
        \vspace{-0.765cm}
        \caption{(b)}
        \label{fig: oracle opt}
    \end{subfigure}
    \caption{\textbf{CloneBO efficiently optimizes antibodies \textit{in silico}.}
    We show the mean and standard deviation of the best acheived value across 10 replicates.
    (a) CloneBO efficiently optimizes a fitness function. The blue line is CloneBO; the grey are LaMBO-Ab, LaMBO, Sapiens, and Greedy.
    (b) CloneBO optimizes binding and stability \textit{in silico} over 100 steps of iterative design (p value is Mann-Whitney).
    It does significantly better than the next best method for binding (p=0.018 Mann-Whitney) and stability (p=0.006 Mann-Whitney).
    } 
    \label{fig: in silico results}
\end{figure}
\paragraph{Oracle of a fitness function of a clonal family}
We first demonstrate the potential of the CloneBO prior to accelerate optimization.
We build an oracle to simulate the fitness function of a real human clonal family, that is, a function from the CloneBO prior.
We trained a language model on a heavy chain clonal family of $10015$ sequences from our test set and try to maximize $f(X)=$ the log likelihood of $X$ of this model.
We start with a single measurement $\hat X_1, Y_1$ where $\hat X_1$ is a sequence from the clonal family.
Very few mutations of an antibody improve fitness, so in Fig.~\ref{fig: clone opt} we see some baselines struggle to identify positive mutations.
CloneBO's prior on the other hand gives it knowledge about what sorts of mutations are most likely to improve fitness allowing it to quickly optimize $f$ even at very low $N$.

\paragraph{Oracles from laboratory measurements of melting temperature and binding}
We next demonstrate the utility of CloneBO in an \textit{in silico} simulation of a realistic in-lab setting.
We trained oracles on lab measurements from an experiment that aimed to optimize a VHH domain, a sequence similar to an antibody heavy chain, for binding and stability at high temperature.
We trained neural network ensembles on the melting temperature and binding (measured in $-\log K_D$) measurements and try to maximize $f(X)=$ the mean predictions of these ensembles.
We simulate starting part way through this experiment by starting with 1000 measurements $\hat X_{1:1000}, Y_{1:1000}$ where $\hat X_{1:1000}$ are the first $1000$ sequences measured in the experiment and $ Y_{1:1000}$ are oracle predictions.
We do not expect mutations outside of the CDR regions of an antibody to substantially affect binding, so we only allow mutations in these regions of the sequence when optimizing for binding.
In Fig.~\ref{fig: oracle opt}, we see that after 100 steps, greedy methods and methods informed by an antibody prior optimize antibodies more efficiently than previous methods against these oracles; in particular, CloneBO outperforms all baselines.
In App.~\ref{sec: against N} we plot these results against $N$ and in table form.

\paragraph{Comparison to structure-based design model for binding SARS CoV.}
In App.~\ref{sec: cov binding} we show that CloneBO can also efficiently optimize antibodies \textit{in silico} for SARS CoV binding as measured by a predictor trained on CoVAbDaB~\citep{Raybould2020-us}.
In particular we see CloneBO beats structure-based design method DiffAb~\citep{Luo2022-ha}.

\paragraph{Ablations and sensitivity}
In App.~\ref{sec: ablations} we show that CloneBO accelerates optimization by building an accurate posterior -- the performance of CloneBO is harmed when we ablate sampling large clonal families, conditioning on experimental data, or our twisted SMC sampling strategy.
We also perform two other ablations demonstrating that our results above are reliable -- we show 1) CloneBO is robust to different starting sequences and starting pool sizes, and 2) CloneBO can efficiently optimize antibodies when $f$ deviates from the CloneBO prior, especially at low $N$.
Finally we sweep hyperparameters to show that CloneBO is not particularly sensitive to hyperparameters other than the amount of noise in the data;
we describe a heuristic that allows us to make a good choice for the amount of noise.

\subsection{Optimizing an antibody \textit{in vitro}}\label{sec: lab results}

\begin{figure}
    \centering
    \begin{subfigure}[b]{0.33\textwidth}
        \includegraphics[width=\linewidth]{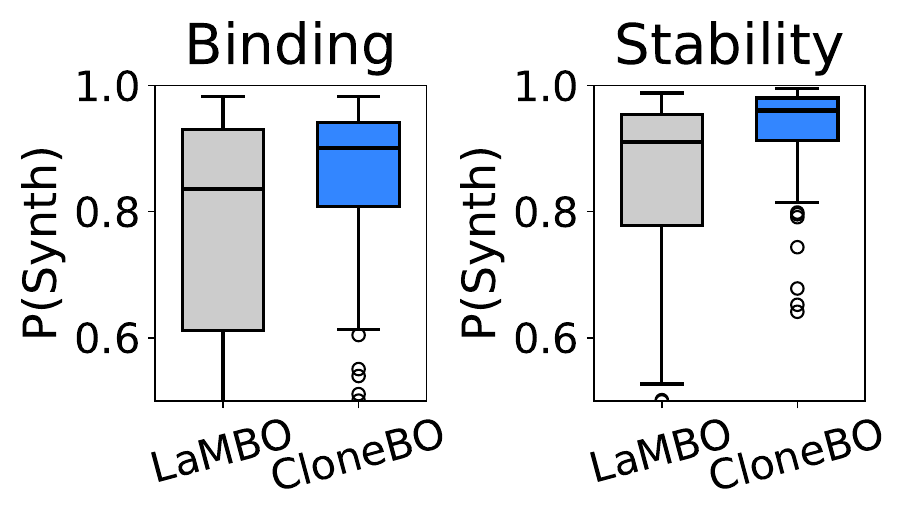}
        \caption{}  
        \label{fig:synth pred}
    \end{subfigure}
    \begin{subfigure}[b]{0.66\textwidth}
        \includegraphics[width=0.495\textwidth]{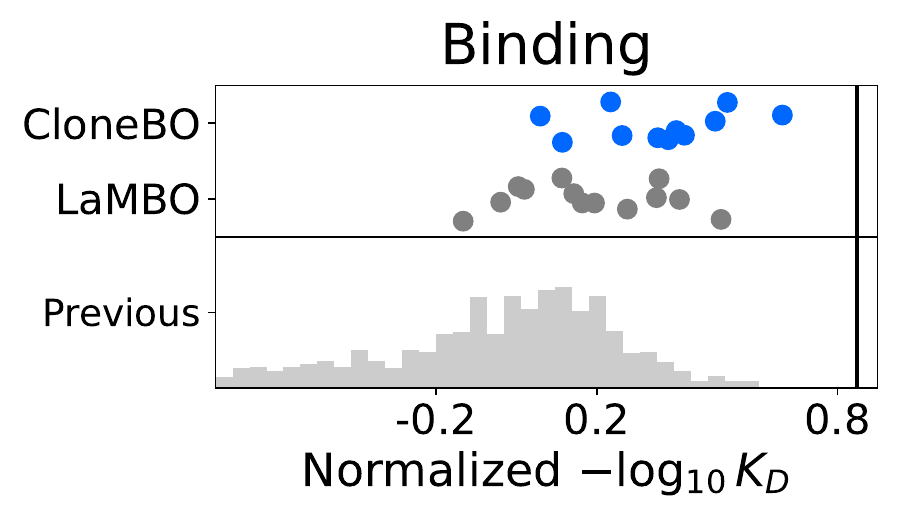}
        \includegraphics[width=0.495\textwidth]{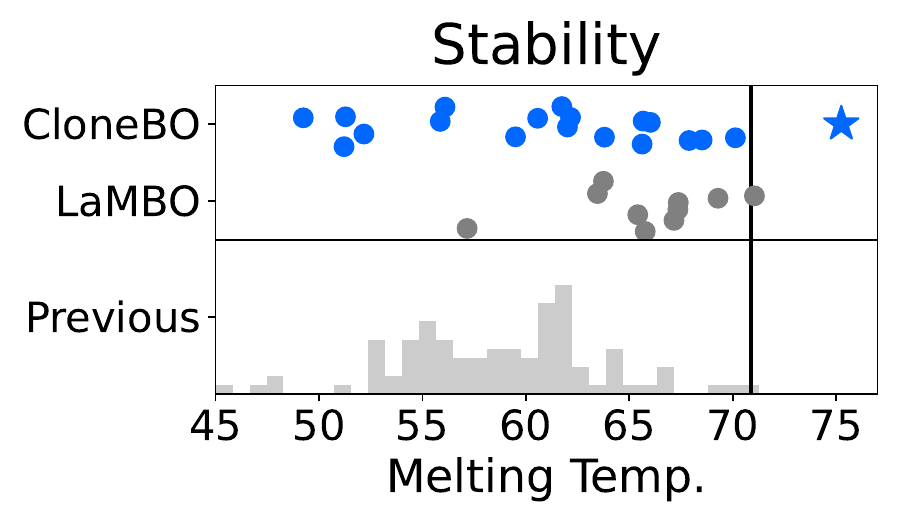}
        \caption{}
        \label{fig: lab opt}
    \end{subfigure}
    \caption{\textbf{CloneBO efficiently optimizes antibodies \textit{in vitro}.}
    \emph{LaMBO} in this plot refers to \emph{LaMBO-Ab}.
    (a) CloneBO design sequences predicted to be synthesizable.
    (b) CloneBO designs strong and stable binders in the lab. 
    Measurements from previous rounds are shown in a histogram.
    The vertical black line represents the best value previously achieved.
    } 
    \label{fig: in vitro results}
\end{figure}
We now demonstrate that CloneBO can design sequences as part of a real-world antibody optimization campaign.
We started with 1000 lab measurements of binding and melting temperature\footnote{94 / 1000 sequences were measured for melting temperature and 997 / 1000 were measured for $K_D$.} (visualized in Fig.~\ref{fig:visualize train}) and designed 200 sequences using CloneBO and our strongest baseline, LaMBO-Ab, for one wetlab iteration of optimizing binding or melting temperature.
In a real world optimization campaign, this one step would be repeated many times.

Before measuring designed sequences, sequences need to be synthesized;
    sequences which are atypical can fail to synthesize, making their measurement impossible.
In Fig.~\ref{fig:synth pred} we plot the predicted synthesizability of sequences from CloneBO and LaMBO-Ab;
sequences from CloneBO are significantly more synthesizable (Mann-Whitney $p<1e-5$), suggesting they are more realistic.
We next measure 20 sequences designed by CloneBO and LaMBO-Ab that are predicted to synthesize.

In Fig.~\ref{fig: lab opt} we plot sequences we were able to measure; 
we include any sequences that were proposed that we had measured in previous experiments. 
We see that sequences from CloneBO achieve the best binding and stability.
Our strongest binder is only beaten by 2 / 997 previously measured sequences and our most stable sequence beats the previously measured sequences by a large margin. 
We also conclude that sequences from CloneBO are significantly stronger binders than those from LamBO-Ab (Mann-Whitney $p=0.021$). We discuss these results in more detail in App.~\ref{sec: additional in vitro}.

\section{Conclusion}\label{sec: conclusion}

To develop new disease treatments, antibodies must be optimized for a range of properties.
By learning from the human immune system's approach to antibody maturation and therefore substantially accelerating optimization in the lab, CloneBO can help build safer and more effective therapeutics.

An important direction of future work is addressing theoretical and practical limitations of CloneBO.
First, CloneBO currently assumes a simple relationship between the fitness of a clonal family and measurements in the lab (Section~\ref{sec: lab prior}).
Future work may account for heteroskedasticity or nonlinear relationships.
As well, CloneBO evaluates the fitness of a sequence by assessing how likely it is to belong to a clonal family of $X_0$.
Future work may attempt to incorporate patterns learned from measurements of diverse sequences which are unlikely to belong to the same clonal family.
Finally, the cost of sampling from the CloneBO posterior scales with the number of laboratory measurements $N$ (Section~\ref{sec: tsmc}), so CloneBO scales by conditioning only a subset of measured sequences.
Future work could instead build a more scalable model or approximate sampling procedure.

Another important future direction is extending CloneBO to multi-objective iterative Bayesian optimization of antibodies for binding and stability simultaneously.
We can do so for example by swapping Thompson sampling for other acquisition functions \citep{Daulton2020-nb}.

\section{Reproducibility Statement}
We include weights for our CloneBO models and code for implementing sampling in our code release.
We describe how to train a CloneBO model, including the parameters that were used to build the training data, in the appendix.
Sequences from the iterative optimization experiment are proprietary; we release all other data: 
    we include the trained oracle for Fig.~\ref{fig: clone opt} in our code release;
    we describe how implement baselines and the oracle in Fig.~\ref{fig:cov validation} in the appendix.

\subsection*{Acknowledgements}
This work is supported by NSF CAREER IIS-2145492, NSF CDS\&E-MSS 2134216, NSF HDR-
2118310, Capital One, and an Amazon Research Award.

\bibliographystyle{abbrvnat}
\bibliography{refs}

\begin{thebibliography}{52}
\providecommand{\natexlab}[1]{#1}
\providecommand{\url}[1]{\texttt{#1}}
\expandafter\ifx\csname urlstyle\endcsname\relax
  \providecommand{\doi}[1]{doi: #1}\else
  \providecommand{\doi}{doi: \begingroup \urlstyle{rm}\Url}\fi

\bibitem[Angermueller et~al.(2020)Angermueller, Dohan, Belanger, Deshpande,
  Murphy, and Colwell]{Angermueller2020-pa}
C.~Angermueller, D.~Dohan, D.~Belanger, R.~Deshpande, K.~Murphy, and L.~J.
  Colwell.
\newblock Model-based reinforcement learning for biological sequence design.
\newblock \emph{Int Conf Learn Represent}, Apr. 2020.

\bibitem[Brookes et~al.(2019)Brookes, Park, and Listgarten]{Brookes2019-fv}
D.~Brookes, H.~Park, and J.~Listgarten.
\newblock Conditioning by adaptive sampling for robust design.
\newblock In K.~Chaudhuri and R.~Salakhutdinov, editors, \emph{Proceedings of
  the 36th International Conference on Machine Learning}, volume~97 of
  \emph{Proceedings of Machine Learning Research}, pages 773--782. PMLR, 2019.

\bibitem[Burnett et~al.(2018)Burnett, Langley, Schofield, Hermes, Chan,
  Jackson, Bourne, Reed, Patterson, Porebski, Brink, Christ, and
  Goodnow]{Burnett2018-fo}
D.~L. Burnett, D.~B. Langley, P.~Schofield, J.~R. Hermes, T.~D. Chan,
  J.~Jackson, K.~Bourne, J.~H. Reed, K.~Patterson, B.~T. Porebski, R.~Brink,
  D.~Christ, and C.~C. Goodnow.
\newblock Germinal center antibody mutation trajectories are determined by
  rapid self/foreign discrimination.
\newblock \emph{Science}, 360\penalty0 (6385):\penalty0 223--226, Apr. 2018.

\bibitem[Carter and Lazar(2018)]{carter2018next}
P.~J. Carter and G.~A. Lazar.
\newblock Next generation antibody drugs: pursuit of the'high-hanging fruit'.
\newblock \emph{Nature Reviews Drug Discovery}, 17\penalty0 (3):\penalty0
  197--223, 2018.

\bibitem[Daulton et~al.(2020)Daulton, Balandat, and Bakshy]{Daulton2020-nb}
S.~Daulton, M.~Balandat, and E.~Bakshy.
\newblock Differentiable expected hypervolume improvement for parallel
  multi-objective bayesian optimization.
\newblock In \emph{Proceedings of the 34th International Conference on Neural
  Information Processing Systems}, number Article 826 in NIPS '20, pages
  9851--9864, Red Hook, NY, USA, Dec. 2020. Curran Associates Inc.

\bibitem[Dopp and Reuel(2020)]{dopp2020simple}
J.~L. Dopp and N.~F. Reuel.
\newblock Simple, functional, inexpensive cell extract for in vitro prototyping
  of proteins with disulfide bonds.
\newblock \emph{Biochemical Engineering Journal}, 164:\penalty0 107790, 2020.

\bibitem[Falck et~al.(2024)Falck, Wang, and Holmes]{Falck2024-fm}
F.~Falck, Z.~Wang, and C.~C. Holmes.
\newblock Are large language models bayesian? a martingale perspective on
  {In-Context} learning.
\newblock In \emph{{ICLR} 2024 Workshop on Secure and Trustworthy Large
  Language Models}, Apr. 2024.

\bibitem[Fannjiang et~al.(2022)Fannjiang, Bates, Angelopoulos, Listgarten, and
  Jordan]{Fannjiang2022-jr}
C.~Fannjiang, S.~Bates, A.~N. Angelopoulos, J.~Listgarten, and M.~I. Jordan.
\newblock Conformal prediction for the design problem.
\newblock Feb. 2022.

\bibitem[Fong et~al.(2024)Fong, Holmes, and Walker]{Fong2024-bx}
E.~Fong, C.~Holmes, and S.~G. Walker.
\newblock Martingale posterior distributions.
\newblock \emph{J. R. Stat. Soc. Series B Stat. Methodol.}, 85\penalty0
  (5):\penalty0 1357--1391, Feb. 2024.

\bibitem[Frazer et~al.(2021)Frazer, Notin, Dias, Gomez, Min, Brock, Gal, and
  Marks]{Frazer2021-hn}
J.~Frazer, P.~Notin, M.~Dias, A.~Gomez, J.~K. Min, K.~Brock, Y.~Gal, and D.~S.
  Marks.
\newblock Disease variant prediction with deep generative models of
  evolutionary data.
\newblock \emph{Nature}, 599\penalty0 (7883):\penalty0 91--95, Nov. 2021.

\bibitem[Frazier(2018)]{Frazier2018-ze}
P.~I. Frazier.
\newblock A tutorial on bayesian optimization.
\newblock July 2018.

\bibitem[Gruver et~al.(2023)Gruver, Stanton, Frey, Rudner, Hotzel,
  Lafrance-Vanasse, Rajpal, Cho, and Wilson]{Gruver2023-sf}
N.~Gruver, S.~D. Stanton, N.~C. Frey, T.~G.~J. Rudner, I.~Hotzel,
  J.~Lafrance-Vanasse, A.~Rajpal, K.~Cho, and A.~G. Wilson.
\newblock Protein design with guided discrete diffusion.
\newblock In \emph{Thirty-seventh Conference on Neural Information Processing
  Systems}, Nov. 2023.

\bibitem[Hansen and Ostermeier(2001)]{Hansen2001-dn}
N.~Hansen and A.~Ostermeier.
\newblock Completely derandomized self-adaptation in evolution strategies.
\newblock \emph{Evol. Comput.}, 9\penalty0 (2):\penalty0 159--195, 2001.

\bibitem[Hewitt and Savage(1955)]{Hewitt1955-ua}
E.~Hewitt and L.~J. Savage.
\newblock Symmetric measures on cartesian products.
\newblock \emph{Trans. Amer. Math. Soc.}, 1955.

\bibitem[Hie et~al.(2023)Hie, Shanker, Xu, Bruun, Weidenbacher, Tang, Wu, Pak,
  and Kim]{Hie2023-sr}
B.~L. Hie, V.~R. Shanker, D.~Xu, T.~U.~J. Bruun, P.~A. Weidenbacher, S.~Tang,
  W.~Wu, J.~E. Pak, and P.~S. Kim.
\newblock Efficient evolution of human antibodies from general protein language
  models.
\newblock \emph{Nat. Biotechnol.}, 42\penalty0 (2):\penalty0 275--283, Apr.
  2023.

\bibitem[Jarasch et~al.(2015)Jarasch, Koll, Regula, Bader, Papadimitriou, and
  Kettenberger]{jarasch2015developability}
A.~Jarasch, H.~Koll, J.~T. Regula, M.~Bader, A.~Papadimitriou, and
  H.~Kettenberger.
\newblock Developability assessment during the selection of novel therapeutic
  antibodies.
\newblock \emph{Journal of pharmaceutical sciences}, 104\penalty0 (6):\penalty0
  1885--1898, 2015.

\bibitem[Jiang et~al.(2023)Jiang, Sablayrolles, Mensch, Bamford, Chaplot,
  de~las Casas, Bressand, Lengyel, Lample, Saulnier, Lavaud, Lachaux, Stock,
  Le~Scao, Lavril, Wang, Lacroix, and El~Sayed]{Jiang2023-yy}
A.~Q. Jiang, A.~Sablayrolles, A.~Mensch, C.~Bamford, D.~S. Chaplot, D.~de~las
  Casas, F.~Bressand, G.~Lengyel, G.~Lample, L.~Saulnier, L.~R. Lavaud, M.-A.
  Lachaux, P.~Stock, T.~Le~Scao, T.~Lavril, T.~Wang, T.~Lacroix, and
  W.~El~Sayed.
\newblock Mistral {7B}.
\newblock Oct. 2023.

\bibitem[Jin et~al.(2021)Jin, Wohlwend, Barzilay, and Jaakkola]{Jin2021-ub}
W.~Jin, J.~Wohlwend, R.~Barzilay, and T.~Jaakkola.
\newblock Iterative refinement graph neural network for antibody
  sequence-structure co-design.
\newblock In \emph{International Conference of Learning Representations 2022},
  Oct. 2021.

\bibitem[Kalchbrenner et~al.(2016)Kalchbrenner, Espeholt, Simonyan, van~den
  Oord, Graves, and Kavukcuoglu]{Kalchbrenner2016-ao}
N.~Kalchbrenner, L.~Espeholt, K.~Simonyan, A.~van~den Oord, A.~Graves, and
  K.~Kavukcuoglu.
\newblock Neural machine translation in linear time.
\newblock \emph{arXiv.org}, 2016.

\bibitem[Kong et~al.(2023)Kong, Huang, and Liu]{Kong2023-yu}
X.~Kong, W.~Huang, and Y.~Liu.
\newblock End-to-end full-atom antibody design.
\newblock In \emph{\textit{Proceedings of the }40th \textit{International
  Conference on Machine Learning}}, Jan. 2023.

\bibitem[Lawson et~al.(2023)Lawson, Li, and Linderman]{Lawson2023-hc}
D.~Lawson, M.~Y. Li, and S.~Linderman.
\newblock {NAS-X}: Neural adaptive smoothing via twisting.
\newblock In \emph{Thirty-seventh Conference on Neural Information Processing
  Systems}, Nov. 2023.

\bibitem[Lawson et~al.(2024)Lawson, Ravent{\'o}s, Warrington, and
  Linderman]{Lawson2024-jd}
D.~Lawson, A.~Ravent{\'o}s, A.~Warrington, and S.~Linderman.
\newblock {SIXO}: smoothing inference with twisted objectives.
\newblock In \emph{Proceedings of the 36th International Conference on Neural
  Information Processing Systems}, number Article 2815 in NIPS '22, pages
  38844--38858, Red Hook, NY, USA, Apr. 2024. Curran Associates Inc.

\bibitem[Lee et~al.(2022)Lee, Yun, Nam, Fong, and Lee]{Lee2022-ij}
H.~Lee, E.~Yun, G.~Nam, E.~Fong, and J.~Lee.
\newblock Martingale posterior neural processes.
\newblock In \emph{The Eleventh International Conference on Learning
  Representations}, Sept. 2022.

\bibitem[Lu et~al.(2020)Lu, Hwang, Liu, Lee, Tsai, Li, and Wu]{Lu2020-qy}
R.-M. Lu, Y.-C. Hwang, I.-J. Liu, C.-C. Lee, H.-Z. Tsai, H.-J. Li, and H.-C.
  Wu.
\newblock Development of therapeutic antibodies for the treatment of diseases.
\newblock \emph{J. Biomed. Sci.}, 27\penalty0 (1):\penalty0 1, Jan. 2020.

\bibitem[Luo et~al.(2022)Luo, Su, Peng, Wang, Peng, and Ma]{Luo2022-ha}
S.~Luo, Y.~Su, X.~Peng, S.~Wang, J.~Peng, and J.~Ma.
\newblock Antigen-specific antibody design and optimization with
  diffusion-based generative models for protein structures.
\newblock In \emph{Advances in Neural Information Processing Systems 35}. Cold
  Spring Harbor Laboratory, July 2022.

\bibitem[Mascola and Haynes(2013)]{Mascola2013-dr}
J.~R. Mascola and B.~F. Haynes.
\newblock {HIV-1} neutralizing antibodies: understanding nature's pathways.
\newblock \emph{Immunol. Rev.}, 254\penalty0 (1):\penalty0 225--244, July 2013.

\bibitem[Miller(2018)]{Miller2018-wx}
J.~W. Miller.
\newblock A detailed treatment of doob's theorem.
\newblock \emph{arXiv: Statistics Theory}, 2018.

\bibitem[Notin et~al.(2022)Notin, Dias, Frazer, Marchena-Hurtado, Gomez, Marks,
  and Gal]{Notin2022-qw}
P.~Notin, M.~Dias, J.~Frazer, J.~Marchena-Hurtado, A.~N. Gomez, D.~Marks, and
  Y.~Gal.
\newblock Tranception: Protein fitness prediction with autoregressive
  transformers and inference-time retrieval.
\newblock In K.~Chaudhuri, S.~Jegelka, L.~Song, C.~Szepesvari, G.~Niu, and
  S.~Sabato, editors, \emph{Proceedings of the 39th International Conference on
  Machine Learning}, volume 162 of \emph{Proceedings of Machine Learning
  Research}, pages 16990--17017. PMLR, 2022.

\bibitem[Olsen et~al.(2022{\natexlab{a}})Olsen, Boyles, and
  Deane]{Olsen2022-nt}
T.~H. Olsen, F.~Boyles, and C.~M. Deane.
\newblock Observed antibody space: A diverse database of cleaned, annotated,
  and translated unpaired and paired antibody sequences.
\newblock \emph{Protein Sci.}, 31\penalty0 (1):\penalty0 141--146, Jan.
  2022{\natexlab{a}}.

\bibitem[Olsen et~al.(2022{\natexlab{b}})Olsen, Moal, and Deane]{Olsen2022-hy}
T.~H. Olsen, I.~H. Moal, and C.~M. Deane.
\newblock {AbLang}: an antibody language model for completing antibody
  sequences.
\newblock \emph{Bioinform Adv}, 2\penalty0 (1):\penalty0 vbac046, June
  2022{\natexlab{b}}.

\bibitem[Olsen et~al.(2023)Olsen, Abanades, Moal, and Deane]{Olsen2023-fk}
T.~H. Olsen, B.~Abanades, I.~H. Moal, and C.~M. Deane.
\newblock {KA-Search}, a method for rapid and exhaustive sequence identity
  search of known antibodies.
\newblock \emph{Sci. Rep.}, 13\penalty0 (1):\penalty0 1--11, July 2023.

\bibitem[Phad et~al.(2022)Phad, Pinto, Foglierini, Akhmedov, Rossi, Malvicini,
  Cassotta, Fregni, Bruno, Sallusto, and Lanzavecchia]{Phad2022-bh}
G.~E. Phad, D.~Pinto, M.~Foglierini, M.~Akhmedov, R.~L. Rossi, E.~Malvicini,
  A.~Cassotta, C.~S. Fregni, L.~Bruno, F.~Sallusto, and A.~Lanzavecchia.
\newblock Clonal structure, stability and dynamics of human memory {B} cells
  and circulating plasmablasts.
\newblock \emph{Nat. Immunol.}, 23\penalty0 (7):\penalty0 1076--1085, June
  2022.

\bibitem[Prihoda et~al.(2022)Prihoda, Maamary, Waight, Juan, Fayadat-Dilman,
  Svozil, and Bitton]{Prihoda2022-zy}
D.~Prihoda, J.~Maamary, A.~Waight, V.~Juan, L.~Fayadat-Dilman, D.~Svozil, and
  D.~A. Bitton.
\newblock {BioPhi}: A platform for antibody design, humanization, and humanness
  evaluation based on natural antibody repertoires and deep learning.
\newblock \emph{MAbs}, 14\penalty0 (1):\penalty0 2020203, 2022.

\bibitem[Rapp et~al.(2024)Rapp, Bremer, and Romero]{Rapp2024-mo}
J.~T. Rapp, B.~J. Bremer, and P.~A. Romero.
\newblock Self-driving laboratories to autonomously navigate the protein
  fitness landscape.
\newblock \emph{Nature Chemical Engineering}, 1\penalty0 (1):\penalty0 97--107,
  Jan. 2024.

\bibitem[Raybould et~al.(2020)Raybould, Kovaltsuk, Marks, and
  Deane]{Raybould2020-us}
M.~I.~J. Raybould, A.~Kovaltsuk, C.~Marks, and C.~M. Deane.
\newblock {CoV-AbDab}: the coronavirus antibody database.
\newblock \emph{Bioinformatics}, 37\penalty0 (5):\penalty0 734--735, Aug. 2020.

\bibitem[Richardson et~al.(2021)Richardson, Galson, Kellam, Kelly, Smith,
  Palser, Watson, and Deane]{Richardson2021-xi}
E.~Richardson, J.~D. Galson, P.~Kellam, D.~F. Kelly, S.~E. Smith, A.~Palser,
  S.~Watson, and C.~M. Deane.
\newblock A computational method for immune repertoire mining that identifies
  novel binders from different clonotypes, demonstrated by identifying
  anti-pertussis toxoid antibodies.
\newblock \emph{MAbs}, 13\penalty0 (1):\penalty0 1869406, 2021.

\bibitem[Riesselman et~al.(2018)Riesselman, Ingraham, and
  Marks]{Riesselman2018-nm}
A.~J. Riesselman, J.~B. Ingraham, and D.~S. Marks.
\newblock Deep generative models of genetic variation capture the effects of
  mutations.
\newblock \emph{Nat. Methods}, 15\penalty0 (10):\penalty0 816--822, Oct. 2018.

\bibitem[Russo et~al.(2018)Russo, Van~Roy, Kazerouni, Osband, and
  Wen]{Russo2018-vr}
D.~J. Russo, B.~Van~Roy, A.~Kazerouni, I.~Osband, and Z.~Wen.
\newblock A tutorial on thompson sampling.
\newblock \emph{Foundations and Trends\textregistered{} in Machine Learning},
  11\penalty0 (1):\penalty0 1--96, 2018.

\bibitem[Sella and Hirsh(2005)]{Sella2005-to}
G.~Sella and A.~E. Hirsh.
\newblock The application of statistical physics to evolutionary biology.
\newblock \emph{Proc. Natl. Acad. Sci. U. S. A.}, 102\penalty0 (27):\penalty0
  9541--9546, July 2005.

\bibitem[Shin et~al.(2021)Shin, Riesselman, Kollasch, McMahon, Simon, Sander,
  Manglik, Kruse, and Marks]{Shin2021-lc}
J.-E. Shin, A.~J. Riesselman, A.~W. Kollasch, C.~McMahon, E.~Simon, C.~Sander,
  A.~Manglik, A.~C. Kruse, and D.~S. Marks.
\newblock Protein design and variant prediction using autoregressive generative
  models.
\newblock Apr. 2021.

\bibitem[Shuai et~al.(2023)Shuai, Ruffolo, and Gray]{Shuai2022-gp}
R.~W. Shuai, J.~A. Ruffolo, and J.~J. Gray.
\newblock {IgLM}: Infilling language modeling for antibody sequence design.
\newblock \emph{Cell Syst.}, 14\penalty0 (11):\penalty0 979--989.e4, Nov. 2023.

\bibitem[Sinai et~al.(2020)Sinai, Wang, Whatley, Slocum, Locane, and
  Kelsic]{sinai2020adalead}
S.~Sinai, R.~Wang, A.~Whatley, S.~Slocum, E.~Locane, and E.~Kelsic.
\newblock Adalead: A simple and robust adaptive greedy search algorithm for
  sequence design.
\newblock \emph{arXiv preprint}, 2020.

\bibitem[Stanton et~al.(2022)Stanton, Maddox, Gruver, Maffettone, Delaney,
  Greenside, and Wilson]{Stanton2022-ru}
S.~Stanton, W.~Maddox, N.~Gruver, P.~Maffettone, E.~Delaney, P.~Greenside, and
  A.~G. Wilson.
\newblock Accelerating {B}ayesian optimization for biological sequence design
  with denoising autoencoders.
\newblock In K.~Chaudhuri, S.~Jegelka, L.~Song, C.~Szepesvari, G.~Niu, and
  S.~Sabato, editors, \emph{Proceedings of the 39th International Conference on
  Machine Learning}, volume 162 of \emph{Proceedings of Machine Learning
  Research}, pages 20459--20478. PMLR, 2022.

\bibitem[Touvron et~al.(2023)Touvron, Martin, Stone, Albert, Almahairi, Babaei,
  Bashlykov, Batra, Bhargava, Bhosale, Bikel, Blecher, Ferrer, Chen, Cucurull,
  Esiobu, Fernandes, Fu, Fu, Fuller, Gao, Goswami, Goyal, Hartshorn, Hosseini,
  Hou, Inan, Kardas, Kerkez, Khabsa, Kloumann, Korenev, Koura, Lachaux, Lavril,
  Lee, Liskovich, Lu, Mao, Martinet, Mihaylov, Mishra, Molybog, Nie, Poulton,
  Reizenstein, Rungta, Saladi, Schelten, Silva, Smith, Subramanian, Tan, Tang,
  Taylor, Williams, Kuan, Xu, Yan, Zarov, Zhang, Fan, Kambadur, Narang,
  Rodriguez, Stojnic, Edunov, and Scialom]{touvron2023llama}
H.~Touvron, L.~Martin, K.~Stone, P.~Albert, A.~Almahairi, Y.~Babaei,
  N.~Bashlykov, S.~Batra, P.~Bhargava, S.~Bhosale, D.~Bikel, L.~Blecher, C.~C.
  Ferrer, M.~Chen, G.~Cucurull, D.~Esiobu, J.~Fernandes, J.~Fu, W.~Fu,
  B.~Fuller, C.~Gao, V.~Goswami, N.~Goyal, A.~Hartshorn, S.~Hosseini, R.~Hou,
  H.~Inan, M.~Kardas, V.~Kerkez, M.~Khabsa, I.~Kloumann, A.~Korenev, P.~S.
  Koura, M.-A. Lachaux, T.~Lavril, J.~Lee, D.~Liskovich, Y.~Lu, Y.~Mao,
  X.~Martinet, T.~Mihaylov, P.~Mishra, I.~Molybog, Y.~Nie, A.~Poulton,
  J.~Reizenstein, R.~Rungta, K.~Saladi, A.~Schelten, R.~Silva, E.~M. Smith,
  R.~Subramanian, X.~E. Tan, B.~Tang, R.~Taylor, A.~Williams, J.~X. Kuan,
  P.~Xu, Z.~Yan, I.~Zarov, Y.~Zhang, A.~Fan, M.~Kambadur, S.~Narang,
  A.~Rodriguez, R.~Stojnic, S.~Edunov, and T.~Scialom.
\newblock Llama 2: Open foundation and fine-tuned chat models, 2023.

\bibitem[Trippe et~al.(2023)Trippe, Wu, Naesseth, Blei, and
  Cunningham]{Trippe2023-ej}
B.~L. Trippe, L.~Wu, C.~A. Naesseth, D.~Blei, and J.~P. Cunningham.
\newblock Practical and asymptotically exact conditional sampling in diffusion
  models.
\newblock July 2023.

\bibitem[Truong and Bepler(2023)]{Truong2023-dk}
T.~F. Truong, Jr and T.~Bepler.
\newblock {PoET}: A generative model of protein families as
  sequences-of-sequences.
\newblock In \emph{Thirty-seventh Conference on Neural Information Processing
  Systems}, Nov. 2023.

\bibitem[Wang et~al.(2023)Wang, Hu, and Zhang]{Wang2023-vj}
K.~Wang, X.~Hu, and J.~Zhang.
\newblock Fast clonal family inference from large-scale {B} cell repertoire
  sequencing data.
\newblock \emph{Cell Rep Methods}, 3\penalty0 (10):\penalty0 100601, Oct. 2023.

\bibitem[Weinstein et~al.(2022)Weinstein, Amin, Frazer, and
  Marks]{Weinstein2022-tv}
E.~N. Weinstein, A.~N. Amin, J.~Frazer, and D.~S. Marks.
\newblock Non-identifiability and the blessings of misspecification in models
  of molecular fitness and phylogeny.
\newblock \emph{Adv. Neural Inf. Process. Syst.}, Dec. 2022.

\bibitem[Whiteley and Lee(2014)]{Whiteley2014-db}
N.~Whiteley and A.~Lee.
\newblock Twisted particle filters.
\newblock \emph{aos}, 42\penalty0 (1):\penalty0 115--141, Feb. 2014.

\bibitem[Yang et~al.(2019)Yang, Wu, and Arnold]{Yang2019-xe}
K.~K. Yang, Z.~Wu, and F.~H. Arnold.
\newblock Machine-learning-guided directed evolution for protein engineering.
\newblock \emph{Nat. Methods}, 16\penalty0 (8):\penalty0 687--694, Aug. 2019.

\bibitem[Yang et~al.(2022)Yang, Lu, and Fusi]{yang2022convolutions}
K.~K. Yang, A.~X. Lu, and N.~Fusi.
\newblock Convolutions are competitive with transformers for protein sequence
  pretraining.
\newblock In \emph{ICLR2022 Machine Learning for Drug Discovery}, 2022.

\bibitem[Zhao et~al.(2024)Zhao, Brekelmans, Makhzani, and Grosse]{Zhao2024-au}
S.~Zhao, R.~Brekelmans, A.~Makhzani, and R.~Grosse.
\newblock Probabilistic inference in language models via twisted sequential
  monte carlo.
\newblock In \emph{\textit{Proceedings of the }41st \textit{International
  Conference on Machine Learning}}, Apr. 2024.

\end{thebibliography}

\appendix

\section{Details of CloneLM and CloneBO}
\subsection{Data collection and training CloneLM}\label{app: training data}
\paragraph{Clonal family data}
We downloaded all data units of human single chain data on OAS~\citep{Olsen2022-nt}. For both light and heavy chain data, we put 10\% of these units into a test set, 10\% into a validation set, and 80\% into a train set.
We annotated clonal families in each of these data units with FastBCR~\citep{Wang2023-vj} using the default parameters \texttt{cluster\_thre = 3, overlap\_thre = 0.1, consensus\_thre = 0.8}.
We removed any clonal families with fewer than 25 sequences.
We ended up with 731 thousand heavy chain clonal families for training, 81 thousand for validation, and 96 thousand for testing; and 26 thousand light chain clonal families for training, 4 thousand for validation, and 4 thousand for testing.

\paragraph{Training CloneLM} We trained CloneLM using Mistral 7B as our base architecture \citep{Jiang2023-yy}.
We scale down the model size to $24$ layers of attention blocks with $16$ attention heads and a hidden size of $1024$ across embedding and all intermediate hidden states. 
We set our maximum context size to $2048$.
We end up with a Mistral model containing 377 million parameters.
For training, we use a batch size of $2$, a gradient accumulation step of $4$, and we sweep learning rates over $\{0.0005, 0.00025, 0.0001\}$ using a constant learning rate scheduler. All training is performed on 4 NVIDIA A100-SXM4-80GB GPUs. For human light chain data, we train for 24 hours for $40$ epochs. 
For human heavy chain data, we train for 48 hours for $1$ epoch.

\subsection{Twisted Sequential Monte Carlo}
\label{app:sequential_is}

\paragraph{Twisted distributions}
We approximate the marginal $\tilde p_{M+1}(X_{M+1}^{(:l)}, X_{0:M})$ just as in Section~\ref{sec: approx post} by replacing $F_{1:N}^{M+1}$ in Eqn.~\ref{eq: approx post} with \mbox{$F_{1:N}^{M+1, (:l)} = \sum_{l'=1}^{l}F_n^{M+1, (l')} +F^{M}_n$},
\[\tilde p_{M+1}^{(:l)}(X_{M+1}^{(:l)}, X_{0:M})\propto p(X_{1:M}, X_{M+1}^{(:l)}|X_0)p(Y_{1:N}|F^{M+1, (:l)}_{1:N}).\]
We call this approximation the ``twisted'' distribution.

If we pretend that these twisted distributions are exact marginals and that $\tilde p_M$ is also the marginal of $\tilde p_{M+1}$, we can therefore sample each sequence letter-by-letter according to
\begin{equation}\label{eqn: smc update}
    X_{M+1}^{(l+1)}\sim \tilde p_{M+1}^{(:l+1)}(X_{M+1}^{(l+1)}|X_{M+1}^{(:l)}, X_{0:M})\propto p(X_{M+1}^{(l+1)}|X_{M+1}^{(:l)}, X_{0:M})p(Y_{1:N}|F^{M+1, (:l+1)}_{1:N}).
\end{equation}
The first term in Eqn.~\ref{eqn: smc update} samples the next letter according to the unconditional distribution.
The second term is a bias that upweights letters $X_{d, M+1}^{(l+1)}$ for which $F_{1:N}^{M+1, (:l+1)}$ correlates with $Y_{1:N}$;
    this usually means upweighting letters that are more likely if sequences that were measured to have high $Y_n$, $\hat X_n$, were included in the clonal family.

\paragraph{Importance sampling}
The twisted distributions are not exactly the marginals.
We can correct for this discrepancy with sequential weighted importance sampling.
Say we have $X_{0:M}, X_{M+1}^{(:l)}$ approximately sampled from $\tilde p_{M+1}^{(:l)}(X_{0:M}, X_{M+1}^{(:l)})$ with importance weight $w^{M+1, (:l)}$.
Then we can calculate the importance weight of $X_{0:M}, X_{M+1}^{(:l+1)}$ by multiplying by the ratio between $\tilde p_{M+1}^{(:l+1)}(X_{0:M}, X_{M+1}^{(:l+1)})$ and $\tilde p_{M+1}^{(:l+1)}(X_{M+1}^{(l+1)}|X_{M+1}^{(:l)}, X_{0:M})\tilde p_{M+1}^{(:l)}(X_{0:M}, X_{M+1}^{(:l)})$, so,
\begin{equation*}
    \begin{aligned}\frac{w^{M+1, (:l+1)}}{w^{M+1, (:l)}}
    \propto\frac{p(X_{M+1}^{(l+1)}|X_{M+1}^{(:l)}, X_{0:M})}{ \tilde p_{M+1}^{(:l+1)}(X_{M+1}^{(l+1)}|X_{M+1}^{(:l)}, X_{0:M})}\times\frac{p(Y_{1:N}|F_{1:N}^{M+1, (:l+1)})}{ p(Y_{1:N}|F_{1:N}^{M+1, (:l)})}.
    \end{aligned}
\end{equation*}
Therefore if we have $D$ samples $X_{0:M+1}^{1:D}$ with weights $w^{M+1}_{1:D}$ then we can approximately sample from $\tilde p_{M+1}$ by sampling $X_{0:M+1}^{d}$ with probability $\tilde w^{M+1}_{d}=\frac{w^{M+1}_{d}}{\sum_{d'} w^{M+1}_{d'}}$.

\paragraph{Sequential Monte Carlo} 
Say we are iteratively sampling and weighting $D$ sets of sequences and the importance weight for one set $w_d^{M, (:l)}$ becomes much smaller than that of the others.
Ideally we wouldn't waste any more compute on sampling the rest of the sequence.
This is the idea of sequential Monte Carlo --- 
    while generating each set of sequence letter-by-letter, every so often, we resample the sets of sequences with probabilities $\tilde w^{M,(:l)}_{1:D}$.
To decide when to resample, we calculate the essential sample size $\sum_d(\tilde w^{M,(:l)}_{d})^2$ and resample when it goes below $\sqrt{D}$, a classic heuristic.
After we resample, we reset the weights $w_d=1/D$.
As $D\to\infty$ we expect to approximate $\tilde p_{M+1}$ arbitrarily well.

We also note when using the predictive distributions of CloneLM, Eqn.~\ref{eqn: f contrib} is an approximation.
The discrepancy comes from the fact that $F^{M+1}_n$ is the conditional probability of $\hat X_n$ as the $M+2$nd sequence while $\tilde F^{M+1}_n:= \sum_{l}F_n^{M+1, (l)} +F^{M}_n$ is the probability of $\hat X_n$ as the $M+1$st sequence.
For $p$, these probabilities are identical, but this may not be exactly the case for CloneLM.
Therefore, once we have sampled all the letters in a sequence $X_M$ then we have sampled from a distribution proportional to
\[p(X_{0:M})p(Y_{1:N}|\tilde F^{M}_{1:N}).\]
Therefore we also resample at this stage after multiplying the importance weight of sample $d$, $w_d^{M}$ by
\[\frac{p(Y_{1:N}|F^{M}_{1:N})}{p(Y_{1:N}|\tilde F^{M}_{1:N})}.\]

\subsection{Experimental Details of CloneBO}\label{app: clonebo details}
Before conditional generation, we normalized $Y_{1:N}$ to
to \mbox{$\tilde Y_{n} = (Y_{n} -\mathrm{start mean})/\mathrm{start std}$} where $\mathrm{start mean}$ and $\mathrm{start std}$ are the mean and standard deviation of the initial dataset $Y_{1:N}$.
In our experiments we used $D=4$ during twisted SMC, and generated clones of size $M=6$.
We tempered $\sigma$ by the maximum number of sequences we conditioned on, i.e. we used $\sigma = \tilde \sigma / \sqrt{75}$  where $\tilde\sigma = 0.25$.
We run each experiment on a single NVIDIA A100 GPU with 80GB of memory; $100$ steps with CloneBO takes roughly 10 hours.

\section{Experimental details}\label{app: exp details}
\subsection{Baselines}
We implemented Sapiens using the code in \url{https://github.com/Merck/Sapiens} under the MIT licence.
We suggested mutations to a sequence by taking the highest likelihood mutation suggested by \texttt{sapiens.predict\_scores} that had not previously been measured.

We implemented LaMBO using the code in \url{https://github.com/samuelstanton/lambo} under the Apache-2.0 licence.
We used default hyperparameters for the masked language model version of LaMBO.
To restrict mutations to the CDRs, we kept sampling mutations from LaMBO until only the CDR was modified.

To build LaMBO-Ab we pretrained the latent space of a LaMBO MLM model on a training set of 100000 antibody sequences.
We built the training set by taking one sequence from each of 100000 random clonal families from the CloneLM training set.

We also compare to two other genetic algorithms with trained surrogates, AdaLead and Genetic; a NN ensemble Bayesopt method, “Evo-BO”; an evolution method, CMA-ES~\citep{Hansen2001-dn}; an RL method, Dyna-PPO~\citep{Angermueller2020-pa}; and an adaptive sampling method, CbAS~\citep{Brookes2019-fv}.
The first three methods are described in \citet{sinai2020adalead}.
We implemented these methods with code from FLEXS~\citep{sinai2020adalead} using the code from \url{https://github.com/samsinai/FLEXS/tree/master} under an Apache-2.0 license; we used default settings for all methods.

\subsection{Training oracles and initializing optimization}
\paragraph{Oracle of a fitness function of a clonal family} We trained an oracle language model adapted from Llama2 \citep{touvron2023llama} on a single reference human heavy chain clone. There are in total $10015$ sequences in the clone and we split them into $90\%$ train, $5\%$ validation, and $5\%$ test sets. Due to the scarcity of our data, we downscaled a Llama2 model with $7$ billion parameters by keeping $12$ hidden layers with hidden state size of $512$. We used $4$ attention heads, $4$ key-value heads, and kept the context size at $2048$. We ended up with a language model containing 50 million parameters. We trained our model on a single NVIDIA A100-SXM4-80GB GPU for $10$ epochs with a batch size of $32$, a gradient accumulation step of $2$, a learning rate of $0.0005$. 
We obtained a perplexity value of $1.2634$ on the validation set. 
Using the clone oracle, we start optimization with 2 randomly chosen sequences from the test set $\hat X_{1:2}$, and predictions from the oracle $Y_{1:2}$.

\paragraph{Oracles from lab measurements of therapeutic Antibodies}\label{app: lab oracle details}
We were provided temperature and binding measurements of 6880 sequences from an iterative optimization experiment performed in the lab.
We aligned these sequences to a reference antibody sequence and trained ensembles of 10 CARP/Bytenet models \citep{Kalchbrenner2016-ao, yang2022convolutions} on the one hot encodings of the aligned sequences to fit the temperature and binding data. The  K$_d$ model ensemble had a measured vs. predicted spearman correlation in crossvalidation of 0.95, while the the  T$_m$ model ensemble had an spearman correlation of 0.72.
We use the mean prediction of the ensemble as $f$.
We start optimization given the measurements of the first 1000 sequences of this experiment $\hat X_{1:1000}$, and predictions from the oracle $Y_{1:1000}$.

\subsection{Training data from iterative optimization}\label{app: lab validation visualize}
\begin{figure}[h]
    \centering
    \begin{subfigure}[b]{0.45\textwidth}
        \includegraphics[width=\textwidth]{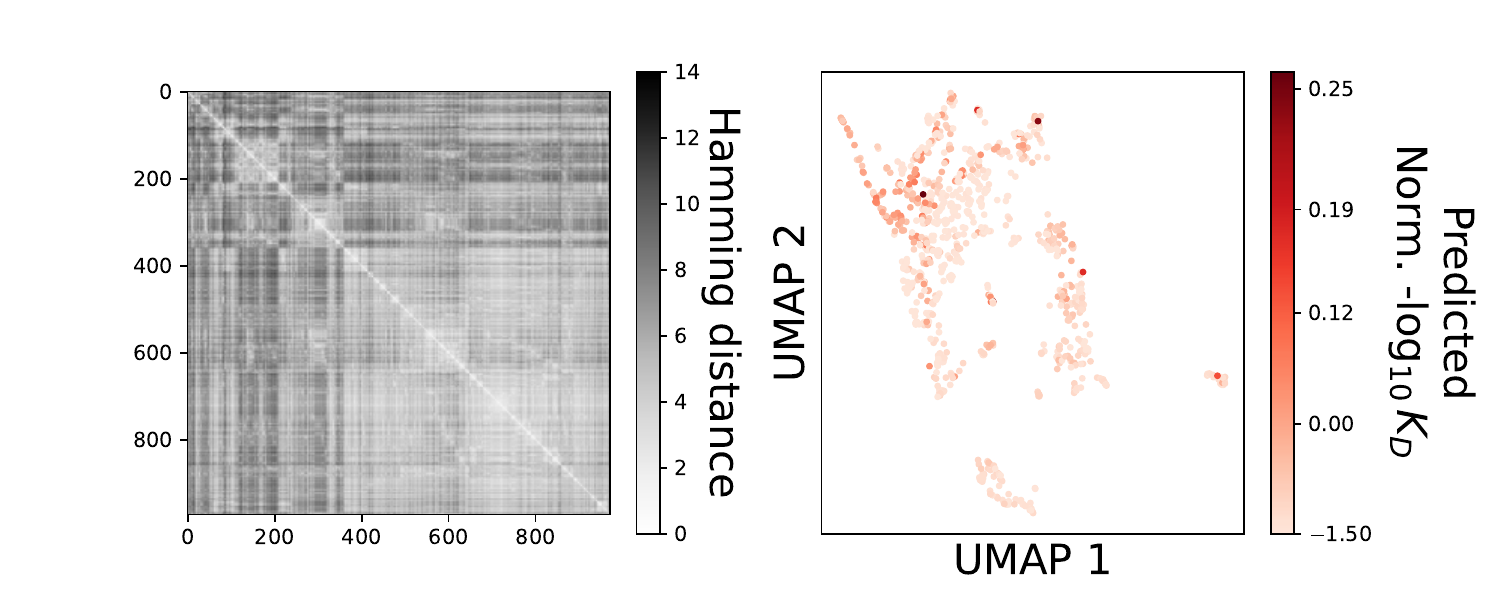}
        \caption{Sequences with binding data.}
    \end{subfigure}
    \begin{subfigure}[b]{0.45\textwidth}
        \includegraphics[width=\textwidth]{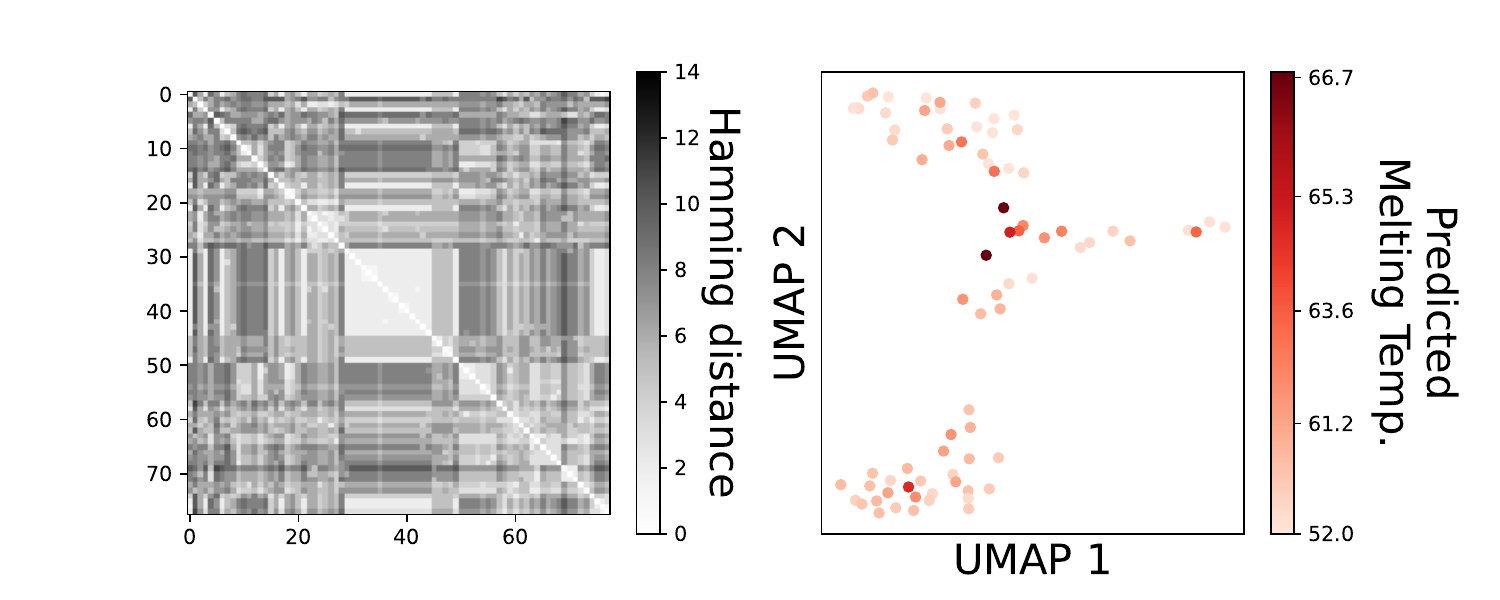}
        \caption{Sequences with stability data.}
    \end{subfigure}
    \caption{Starting pool for (a) binding and (b) stability optimization. We plot the Hamming distance matrices and UMAPs. sequences are coloured by predicted property}
    \label{fig:visualize train}
\end{figure}

\subsection{Lab validation}\label{app: lab validation details}

We built a predictor of synthesizability from measures of expression just the same as predictors of melting temperature and binding in Sec~\ref{app: lab oracle details}.
The predictor achieves a test AUROC of 0.87.

Designed sequences were synthesized via cell-free protein synthesis \citep{dopp2020simple} in 96-well format and purified via Protein A binding on Pierce magnetic beads. Purity and yield were confirmed before further analysis. Affinities ($-\log K_D$) were measured with Bio-Layer Interferometry (BLI) on an Octet instrument, with the antigen immobilized at three different dilutions of antibody. Thermostability (melting temperature) was measured by Nano differential scanning interferometry (NanoDSF) on an Uncle instrument.

\section{Supplementary results}
\subsection{More example generated clonal families}\label{app: clone examples}
In this section we show more light and heavy chain clonal families generated from CloneLM.
In Fig.~\ref{fig: clone ex1} and Fig.s~\ref{fig: light clone ex2}, \ref{fig: light clone ex3}, we see that sequences generated by CloneLM include can introduce insertions and deletions.
The large sets of deletions in the sequences of the clonal family in Fig.~\ref{fig: clone ex1} are due to the fact that some sequences in OAS are missing the beginning or end of their sequences \cite{Olsen2022-hy}.

\begin{figure}[H]
    \centering
    \begin{subfigure}[b]{0.98\textwidth}
        \includegraphics[width=\textwidth]{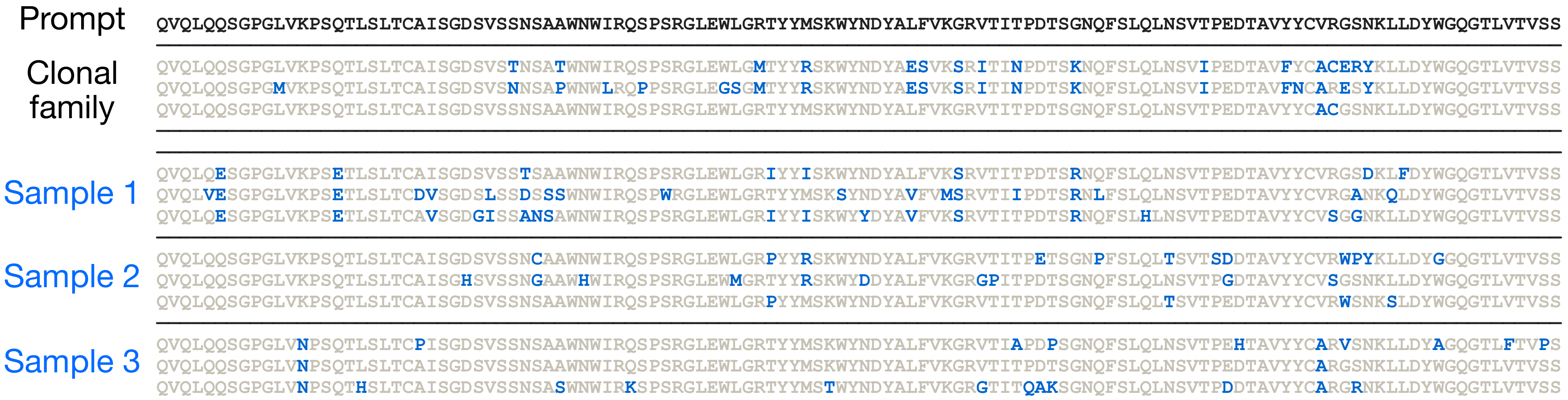}
        \caption{}
        \label{fig: clone ex3}
    \end{subfigure}
    \begin{subfigure}[b]{0.98\textwidth}
        \includegraphics[width=\textwidth]{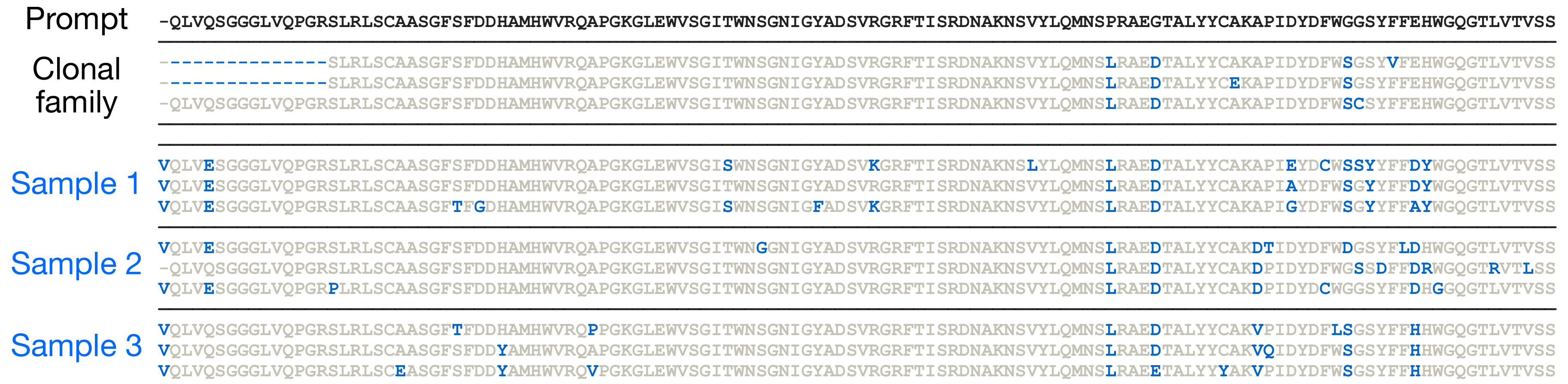}
        \caption{}
        \label{fig: clone ex1}
    \end{subfigure}
    \caption{\textbf{Examples of heavy chain clonal families generated by CloneLM.}}
    \label{fig: heavy clone examples}
\end{figure}

\begin{figure}[H]
    \centering
    \begin{subfigure}[b]{0.98\textwidth}
        \includegraphics[width=\textwidth]{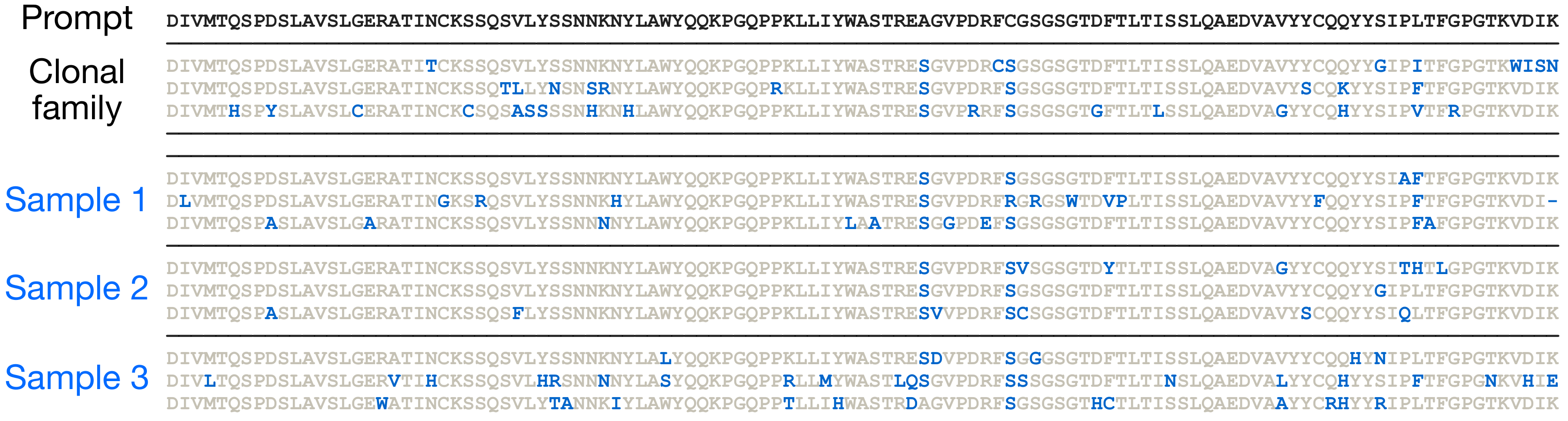}
        \caption{}
        \label{fig: light clone ex1}
    \end{subfigure}
    \begin{subfigure}[b]{0.98\textwidth}
        \includegraphics[width=\textwidth]{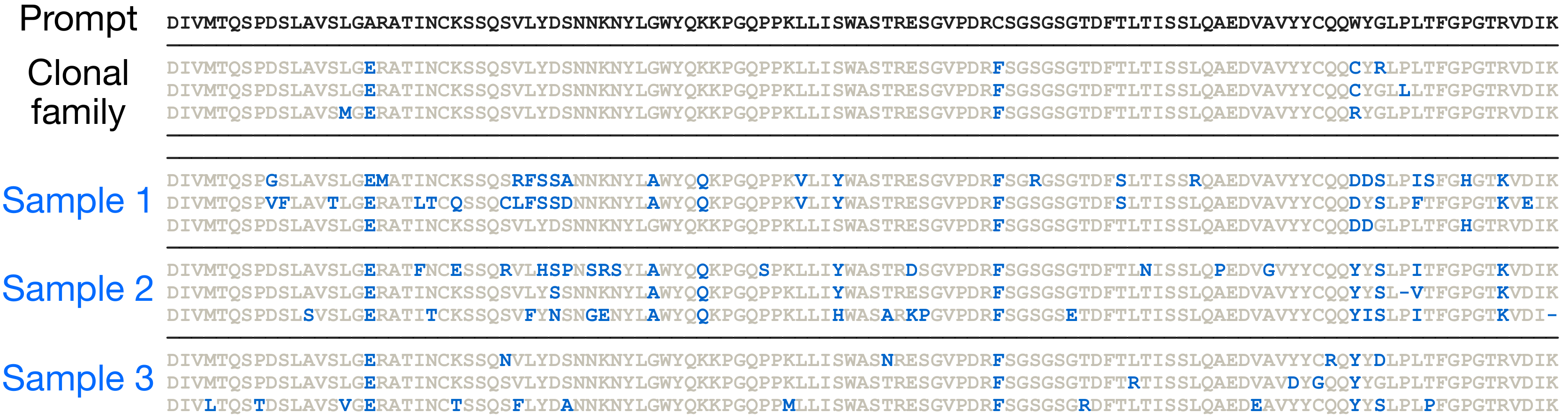}
        \caption{}
        \label{fig: light clone ex2}
    \end{subfigure}
    \begin{subfigure}[b]{0.98\textwidth}
        \includegraphics[width=\textwidth]{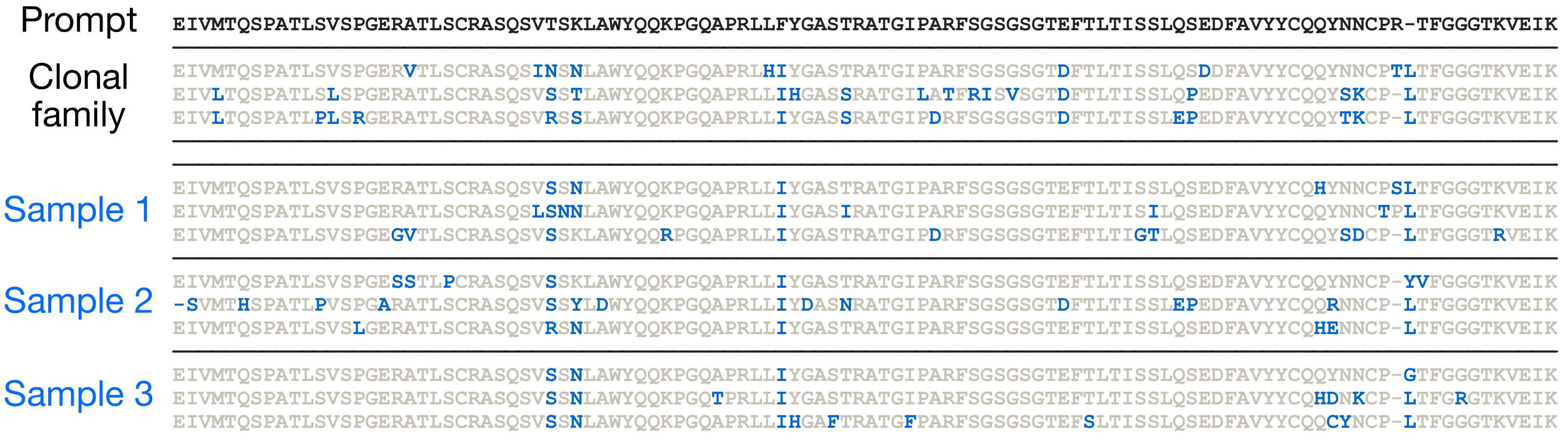}
        \caption{}
        \label{fig: light clone ex3}
    \end{subfigure}
    \caption{\textbf{Examples of light chain clonal families generated by CloneLM.}}
    \label{fig: light clone examples}
\end{figure}

\subsection{Twisted SMC fitting affinity data}\label{app:tsmc affinity res}
In Fig.~\ref{fig:tsmc aff} we show similar results to Fig.~\ref{fig: tsmc} for 75 laboratory measurements of binding.

\begin{figure}[H]
    \centering
    \begin{subfigure}[b]{0.365\textwidth}
        \includegraphics[width=\textwidth]{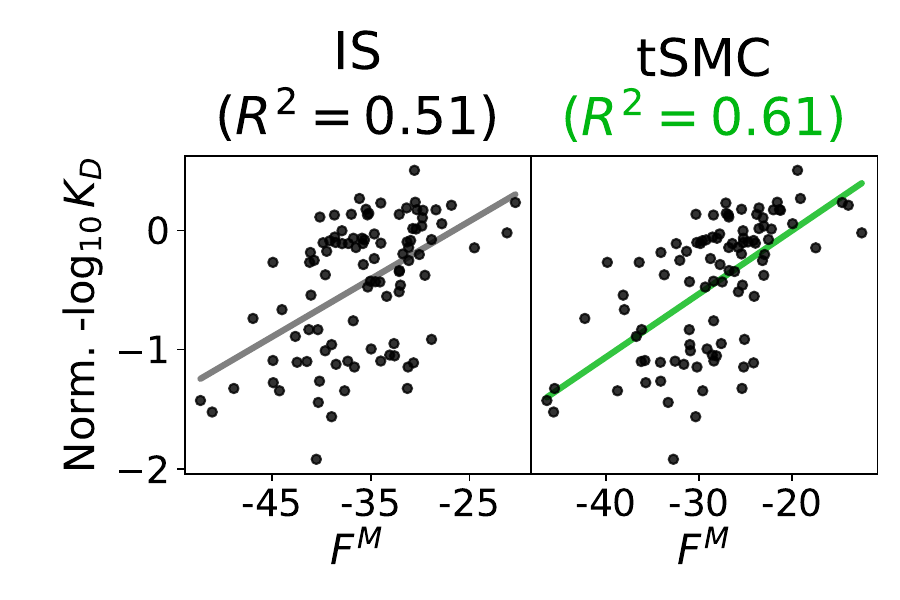}
        \caption{}
        \label{fig: kd fitness fit}
    \end{subfigure}
    \begin{subfigure}[b]{0.305\textwidth}
        \includegraphics[width=\textwidth]{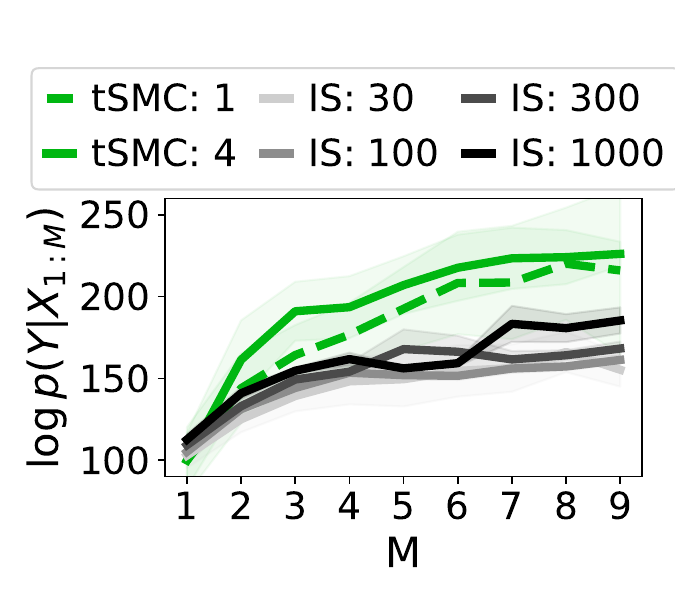}
        \caption{}
        \label{fig: tsmc aff M}
    \end{subfigure}
    \caption{\textbf{Twisted SMC fits laboratory measurements of affinity.} Experiments are similar to those in Fig.~\ref{fig: tsmc}.
    }
    \label{fig:tsmc aff}
\end{figure}

\subsection{Efficient optimization versus number of steps}\label{sec: against N}
\begin{figure}[H]
    \centering
    \includegraphics[width=0.65\linewidth]{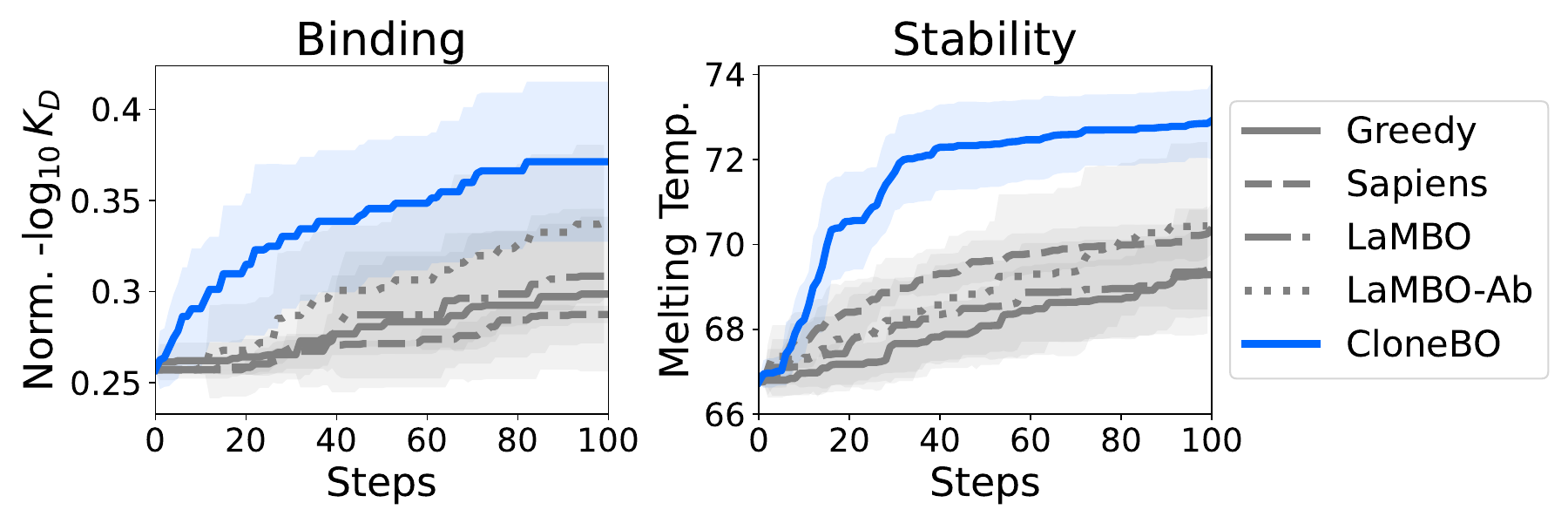}
    \caption{Results of Fig.~\ref{fig: oracle opt} for various $N$ for representative baselines.}
    \label{fig:against N}
\end{figure}

\begin{table}[!h]
\centering
\begin{tabular}{lccc}
\hline
Method & $\log p(X|\mathrm{clone})$ & Norm. $\log_{10} K_D$ & Melting temp. \\
\hline
CbAS & -27.42 & 0.2572 & 66.76 \\
Dyna-PPO & -27.42 & 0.2572 & 66.76 \\
CMA-ES & -27.42 & 0.2853 & 66.76 \\
EvoBO & -27.42 & 0.2572 & 66.76 \\
Genetic & -25.94 & 0.3001 & 70.19 \\
AdaLead & -25.30 & 0.2998 & 69.42 \\
Greedy & -26.65 & 0.2986 & 69.28 \\
LaMBO & -24.99 & 0.3085 & 69.40 \\
Sapiens & -24.51 & 0.2875 & 70.39 \\
LaMBO-Ab & -24.51 & 0.3381 & 70.44 \\
CloneBO & \textbf{-7.29} & \textbf{0.3713} & \textbf{72.92} \\
\hline
\end{tabular}
\caption{Result from Fig.~\ref{fig: in silico results} in table form. We report the mean best value achieved across 10 replicates.}
\end{table}

\subsection{Optimizing antibodies to bind SARS CoV}\label{sec: cov binding}
Above, we validated CloneBO using predictors trained on large-scale mutational data from a lab as well as in \textit{in vitro} experiments – these are the most reliable evaluations of CloneBO.
As another evaluation of CloneBO, and to compare it to structure-based design models,
here we explore optimizing the CDRH3s of antibodies for binding SARS CoV 1 and 2 in humans as measured by predictors used in\citet{Jin2021-ub} (downloaded from \url{https://github.com/wengong-jin/RefineGNN}).
We caution however that while we expect CloneBO to be a good prior for SARS binding, these predictors have large epistemic uncertainty as they are trained on only a few thousand extremely diverse sequences. We thus expect we are optimizing an objective similar to those in Fig.~\ref{fig:fit and N} so that CloneBO should perform best at low N.

These predictors were trained on only functional antibodies and may not generalize outside this set. So, we optimize binding + humanness (measured by IgLM  likelihood); we standardize both binding and humanness to the same variance.

We start with $N=1$ sequence from CoVAbDab. 
In Fig.~\ref{fig:cov validation} we see that for 6 randomly selected starting sequences, CloneBO is consistently among the most efficient methods for CoV1 and, at low N, for CoV2 as well.

We were interested in seeing if structure-based design methods could also be used for iterative design.
When structure is available, we therefore also compare to a state-of-the-art structure method, DiffAb \citep{Luo2022-ha}.
To perform iterative design with DiffAb, we greedily pick one of the 4 best measured sequences and optimize it as in Sec 4.3 of \citet{Luo2022-ha}. 
We see that CloneBO beats this algorithm.

\begin{figure}[!h]
    \centering
    \includegraphics[width=\linewidth]{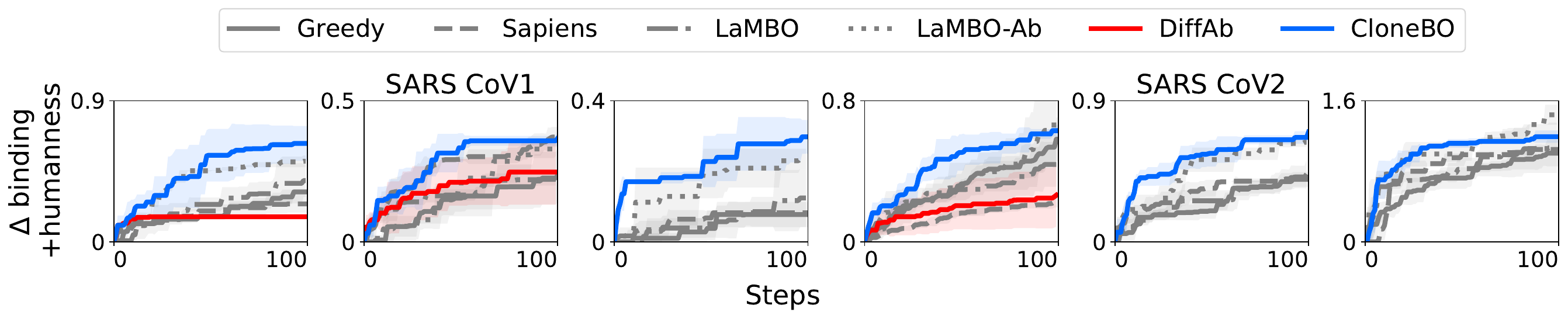}
    \caption{CloneBO efficiently optimizes 6 sequences from CovAbDab for binding to SARS CoV1 (first 3 columns) or CoV2 (last 3 columns). In particular it outperforms a structure-based baseline (DiffAb) for the 3 sequences with an available structure. We show mean and standard deviation achieved across 3 replicates.}
    \label{fig:cov validation}
\end{figure}

\subsection{Ablations of \textit{in silico} optimization}\label{sec: ablations}
\paragraph{CloneBO efficiently optimizes sequences by building an accurate posterior.}
We now show that CloneBO efficiently optimizes sequences by conditioning on experimental measurements and accurately sampling from the martingale posterior.
In Fig.~\ref{fig: ablations} we see optimization is often harmed by 1) not accurately sampling from the martingale posterior by only sampling clones of size $M=1$, 2) using a naive importance sampling procedure instead of twisted SMC, or 3) not conditioning on previous measurements.
We see the $M=1$ or importance sampling ablations have less of an effect when optimizing fitness, potentially because it is easier to condition on data from a function from the CloneBO prior.

\begin{figure}[!h]
    \centering
    \includegraphics[width=0.95\textwidth]{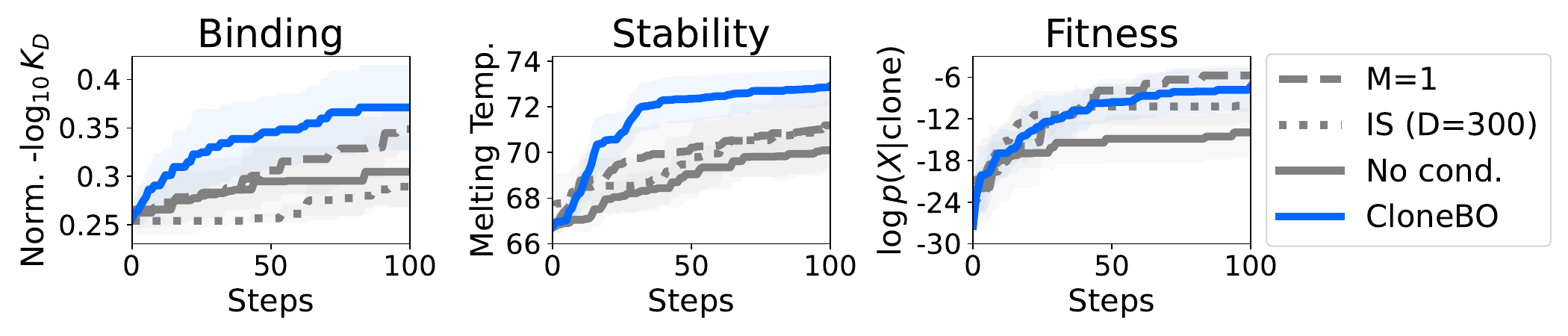}
    \caption{\textbf{CloneBO better optimizes antibodies than models that ablate accurately sampling from the posterior.} We shade a standard deviation across 3 replicates. We show results for optimizing binding \textit{in silico}, stability \textit{in silico}, and the fitness function of a clone.
    }
    \label{fig: ablations}
\end{figure}

\paragraph{CloneBO is robust to different starting pool sizes and deviations from its prior at low N.}

In Fig.~\ref{fig: ablate N} we optimized for stability as in Fig.~\ref{fig: oracle opt} of the paper with a starting pool of various $N$ measurements. 
We see that CloneBO outperforms our baselines across these different starting stabilities and starting pool sizes, especially at low $N$.

In Fig.~\ref{fig: clone opt} in the paper we showed that when optimizing an objective from CloneBO’s prior (the fitness function of a clone), $F$, CloneBO strongly outperforms baselines. Here we mix $F$ with a random neural network, $G$, and optimize $w F + (1-w) G$; $w\in[0, 1]$ controls how well the CloneBO prior describes the objective. We start with $N=2$ sequences. 
In Fig.~\ref{fig: ablate w} we see that in this small $N$ setting, CloneBO outperforms other methods even when the objective only somewhat matches the prior ($w=0.4$).
Even at $w=0.2$, CloneBO is the best method at very low N (up to N=25).

\begin{figure}[H]
    \captionsetup[subfigure]{labelformat=empty}
    \centering
        \begin{subfigure}[b]{0.9\textwidth}
        \makebox[0pt][r]{\raisebox{45pt}{(a)}\hspace{10pt}}
        \includegraphics[width=\textwidth]{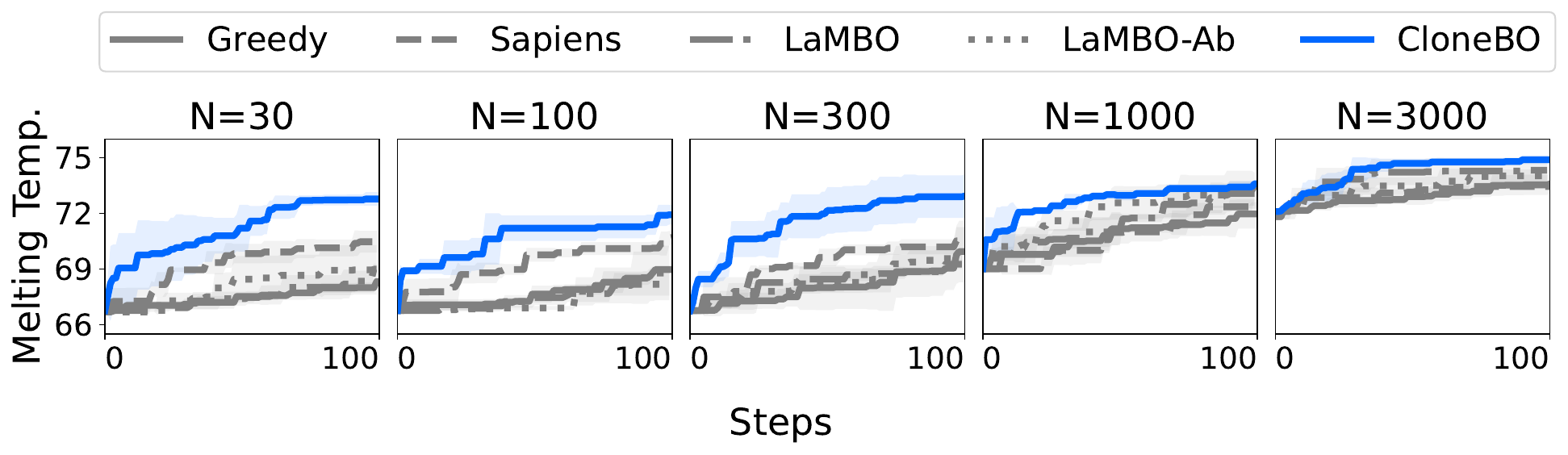}
        \caption{}
        \label{fig: ablate N}
        \vspace{-0.5cm}
    \end{subfigure}
        \begin{subfigure}[b]{0.9\textwidth}
        \makebox[0pt][r]{\raisebox{45pt}{(b)}\hspace{10pt}}
        \includegraphics[width=\textwidth]{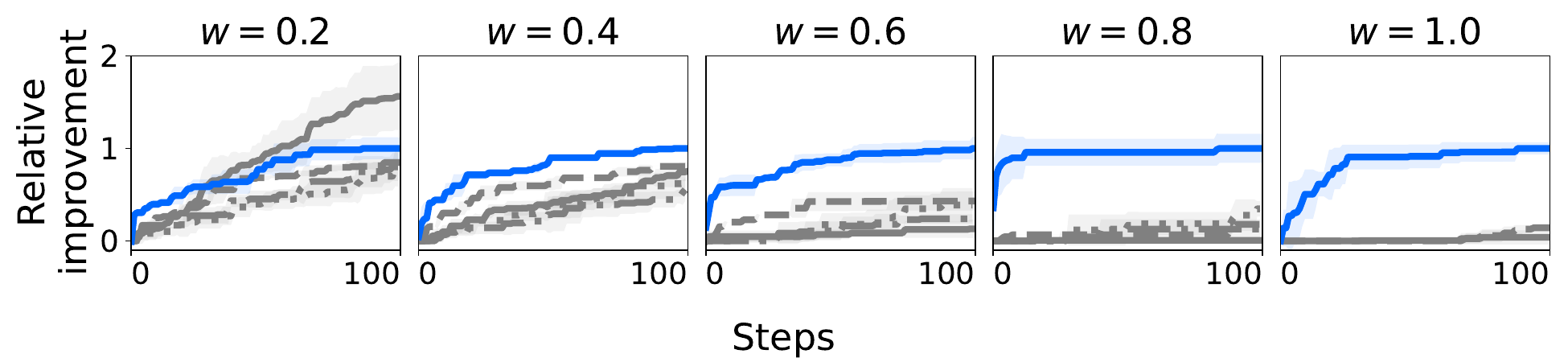}
        \caption{}
        \label{fig: ablate w}
        \vspace{-0.5cm}
    \end{subfigure}
    \caption{\textbf{CloneBO is robust} to ablating (a) the starting pool size or (b) how well its prior describes the objective. We show the mean and standard deviations of the best achieved value across 3 replicates.}
    \label{fig:fit and N}
\end{figure}

\paragraph{Sensitivity to hyperparameters \textit{in silico}}\label{sec: app sensitivity}
We now investigate the sensitivity of our results to 3 hyperparameters: the noise level in the likelihood $\tilde \sigma$, the maximum number of allowed mutations $L$, and the number of sequences we draw $X_0$ from $K$, as described in Sec.~\ref{sec: thompson}.
In our experiments above we use $\log_2\tilde \sigma = -2$, $L=3$, and $K=4$.

We picked $K$ and $L$ based on intuition.
We picked $\tilde\sigma$ by looking at the clones generated when conditioned on the initial binding data \textit{in silico} when suing values $\log_2\tilde \sigma = -4, -3, -2, -1$;
we noted that when $\log_2\tilde \sigma = -4, -3$, generated clones contained very short or long sequences, indicating it was hard to find clones that fit the data with little noise.
Thus we picked $\log_2\tilde \sigma = -2$ to fit the data with as little noise as possible.
Performing the same procedure for the \textit{in silico} stability data, we arrived at the same value $\log_2\tilde \sigma = -2$ was a good choice.
We fixed this value for all our other experiments.

We see in Fig.~\ref{fig: sensitivity sweep K} that $K=1$ performs badly when optimizing binding, likely due to a higher chance of getting stuck in a local minima; otherwise CloneBO is not very sensitive to $K$.
In Fig.~\ref{fig: sensitivity sweep L} we see CloneBO is also not very sensitive to $L$.
In Fig.~\ref{fig: sensitivity sweep sigma} we note CloneBO is sensitive to the choice of $\tilde \sigma$ but our procedure described above picked the optimal $\tilde\sigma$ for binding and stability.
For fitness, we noted the posterior is easy to sample from even at $\log_2\tilde \sigma = -3$ (we see this manifest in Fig.~\ref{fig: ablations} as well) and indeed a smaller value of $\tilde\sigma$ is optimal.
This suggests one can optimize antibodies more efficiently by improving the CloneBO likelihood, potentially by picking a better $\tilde\sigma$ or by adding a prior and marginalizing over it.

\begin{figure}[H]
    \centering
    \begin{subfigure}[b]{0.343\textwidth}
        \includegraphics[width=\textwidth]{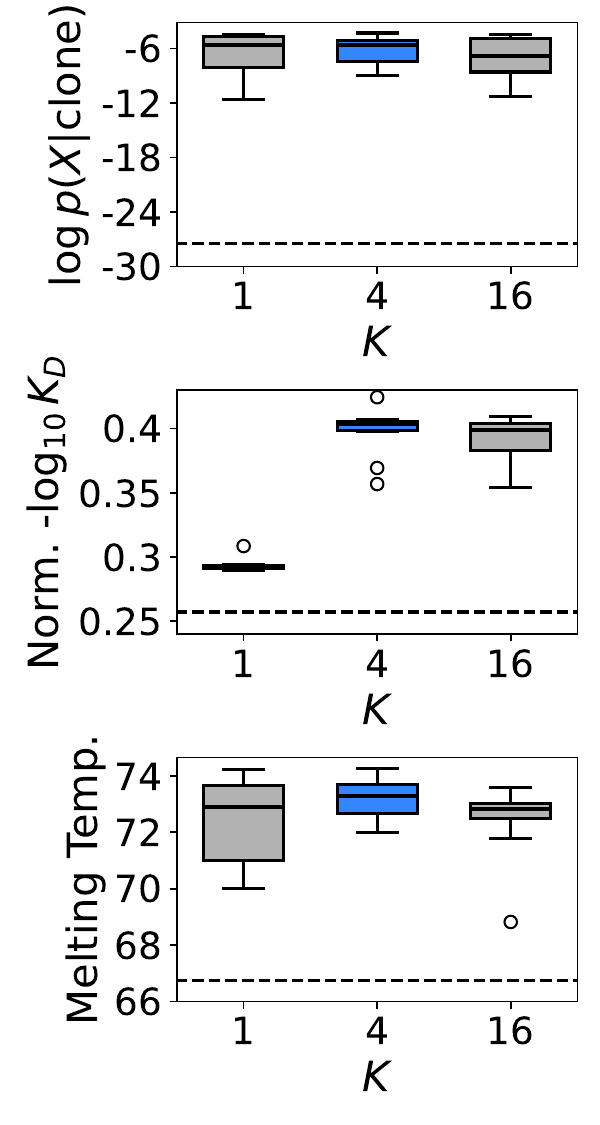}
        \caption{}
        \label{fig: sensitivity sweep K}
    \end{subfigure}
    \begin{subfigure}[b]{0.3\textwidth}
        \includegraphics[width=\textwidth]{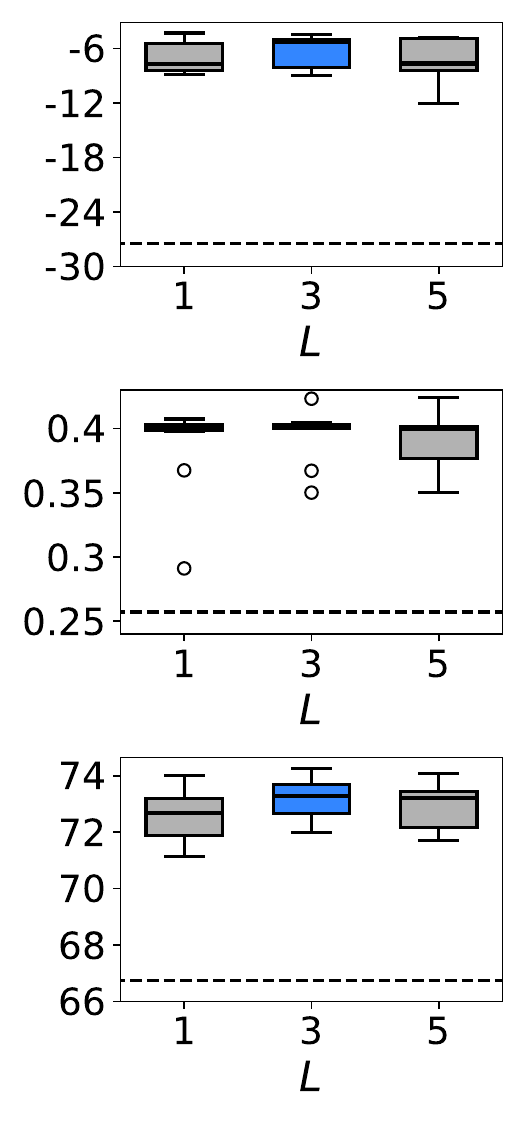}
        \caption{}
        \label{fig: sensitivity sweep L}
    \end{subfigure}
    \begin{subfigure}[b]{0.3\textwidth}
        \includegraphics[width=\textwidth]{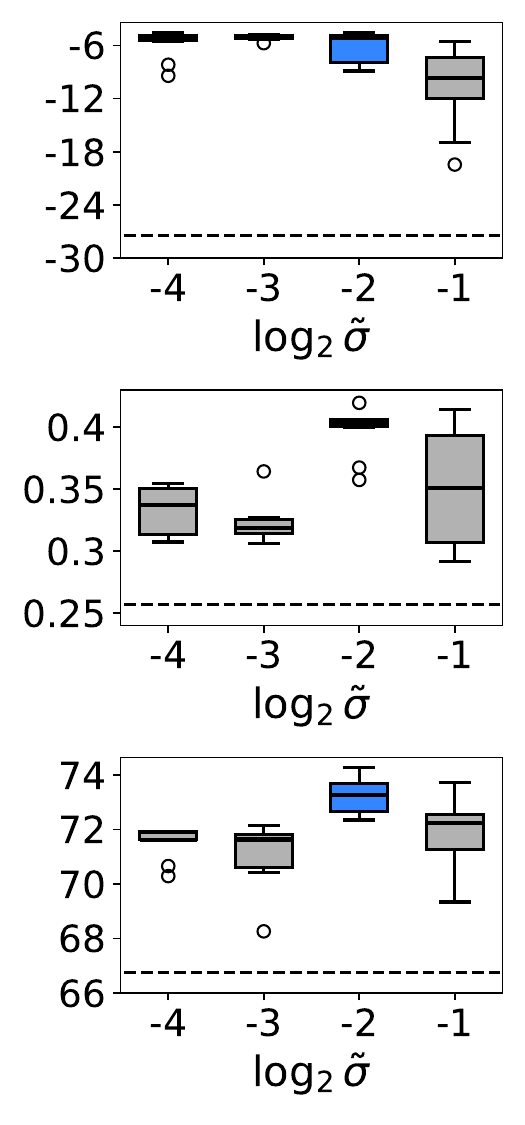}
        \caption{}
        \label{fig: sensitivity sweep sigma}
    \end{subfigure}
    \caption{\textbf{CloneBO sensitivity to hyperparameters for fitness, binding and stability.} Experiments are run with 10 replicates; blue boxes represent hyperparameters used in the main text
. a) Sensitivity to size of pool $X_0$ is randomly selected from, $K$. b) Sensitivity to maximum allowed number of mutations, $L$. c) sensitivity to noise in likelihood, $\tilde\sigma$.  
    }
    \label{fig:sensitivity}
\end{figure}

\subsection{Additional discussion of \textit{in vitro} results}\label{sec: additional in vitro}
We were able to measure the affinities of 9 sequences for CloneBO, and 11 for LaMBO-Ab, and the melting temperatures of 19 sequences from CloneBO and 10 sequences from LaMBO-Ab.
Adding in previously measured sequences, we were able to get two more affinity measurements for CloneBO and 2 more for LaMBO-Ab, and no other affinity measurements.
Note when interpreting these results that affinity and stability may be correlated with dropout and whether or not a sequence was previously measured.
In Fig.~\ref{fig: alt kd measures} we plot the affinity measurements removing the previously measured sequences; 
the results are qualitatively similar to those of Fig.~\ref{fig: in vitro results}.
\begin{figure}[H]
    \centering
    \includegraphics[width=0.4\linewidth]{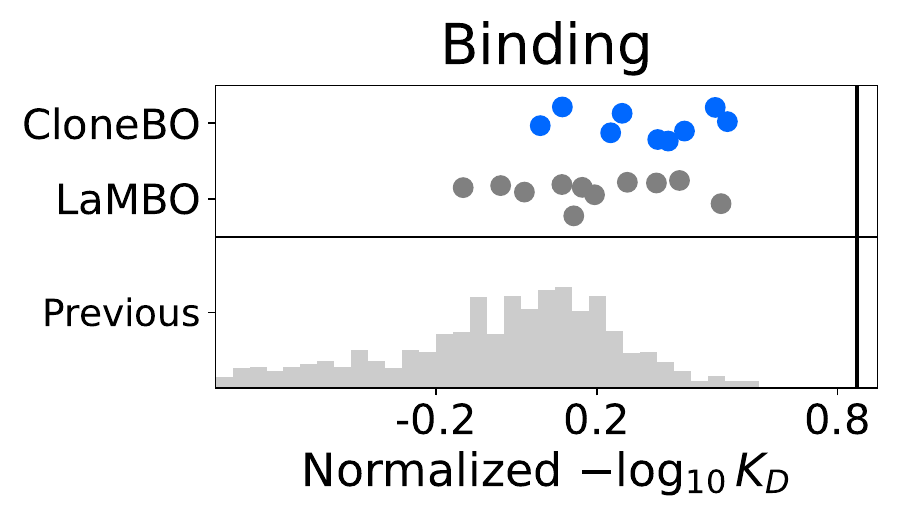}
    \caption{The result of Fig.~\ref{fig: lab opt} for affinity with only sequences that were newly measured.}
    \label{fig: alt kd measures}
\end{figure}

\section{Theoretical results}
\subsection{Analytic form of marginal likelihood}\label{app: marginal lik proof}
\begin{proposition}\label{prop: marginal lik proof}
    \textbf{(Proof of Prop.~\ref{prop: marginal lik})}
    For some constant $C$,
    \[\log p(Y_{1:N}|F_{1:N}) = -\frac 1 2 \log \mathrm{Cov}(F_{1:N})  + \frac 1 2 R^2 + \log\Phi(R)+C'\]
\end{proposition}
\begin{proof}
We first put a wide prior on $M$, $M\sim N(0, \tau^2)$, and then later send $\tau\to \infty$.
Then we can marginalize $M$ out to get 
\begin{gather*}
    Y_n\sim N(\beta F_n, \sigma^2I + n \tau^2e\otimes e)\\
\end{gather*}
where we define $e$ to be the vector with $1/\sqrt{N}$ in each position, $e=\vec {1}/\sqrt{N}$.
Calling $\Sigma = \sigma^2I + n \tau^2e\otimes e$, $\sigma_\beta^2 = (F_{1:N}^T \Sigma^{-1}F_{1:N})^{-1}$, $\mu_\beta = \sigma_\beta^2 F_{1:N}^T \Sigma^{-1}Y_{1:N}$, we get
\begin{equation*}
    \begin{aligned}
        \int_0^\infty  p(Y_{1:N}, \beta|F_{1:N})d\beta \propto & \int_0^\infty  e^{ \beta F_{1:N}^T \Sigma^{-1}Y_{1:N}-\frac 1 2\beta^2 F_{1:N}^T \Sigma^{-1}F_{1:N}}d\beta\\
        =& (2\pi\sigma_\beta^2)^{1/2}e^{\mu_\beta^2/2\sigma_\beta^2}P(N(\mu_\beta, \sigma^2_\beta)>0)\\
        \propto& \sigma_\beta e^{\mu_\beta^2 / {2\sigma_\beta^2}}\Phi\left(\frac{\mu_\beta} {\sigma_\beta}\right).
    \end{aligned}
\end{equation*}

Now, 
\begin{equation*}
    \begin{aligned}
        \Sigma^{-1} =& (\sigma^2(I-e\otimes e) + (\sigma^2 + n\tau^2)e\otimes e)^{-1}\\
        =&\sigma^{-2}(I-e\otimes e) + (\sigma^2 + n\tau^2)^{-1}e\otimes e\\
        \to& \sigma^{-2}(I-e\otimes e)\text{ as }\tau\to\infty\\
        = &N \sigma^{-2}\left(\frac 1 N I-\left(\frac 1 N \vec 1\right)\otimes \left(\frac 1 N \vec 1\right)\right).
    \end{aligned}
\end{equation*}
Therefore, as $\tau\to\infty$, $F_{1:N}^T \Sigma^{-1}F_{1:N}$ is $N \sigma^{-2}$ times the variation of $F_{1:N}$, $\mathrm{Var}(F_{1:N})$, and $F_{1:N}^T \Sigma^{-1}Y_{1:N}$ is $N \sigma^{-2}$ times the covariance of $F_{1:N}$ and $Y_{1:N}$, $\mathrm{Cov}(F_{1:N}, Y_{1:N})$.

Therefore, if we call $\frac{\mu_\beta} {\sigma_\beta}\to \sqrt{N}\sigma^{-1}\mathrm{Cor}(F_{1:N}, Y_{1:N})\mathrm{Std}(Y_{1:N})=R$ then we get that when we send $\tau\to\infty$,
\[\log p(Y_{1:N}|F_{1:N}) = -\frac 1 2 \log \mathrm{Cov}(F_{1:N})  + \frac 1 2 R^2 + \log\Phi(R)+C'\]
for some constant $C'$.
\end{proof}

\subsection{Convergence of approximate posterior}\label{app: approx post proof}

We show that the approximate posteriors converge defined in Eqn.~\ref{eq: approx post} converge to the true posterior.
We make the mild assumption that the hypothetical latent variable $\mathrm{clone}$ has been defined such that $\mathrm{clone}\mapsto p(X|\mathrm{clone})$ is measurable and no two variables $\mathrm{clone}_1$, $\mathrm{clone}_2$ have the same conditional distribution $p(X|\mathrm{clone}_1)\neq p(X|\mathrm{clone}_2)$.
We also assume that all measured sequences are antibodies that can plausibly, although maybe with extremely low likelihood, appear in a clone together, i.e. $p(X_0, \hat X_{1:N})\neq 0$.
Finally, we assume the sequences $\hat X_{1:N}$ are sufficiently diverse so that their log likelihoods cannot all be identical for any $\mathrm{clone}$, i.e. for some $\epsilon>0$, $\mathrm{Cov}(F_{1:N})>\epsilon$ with probability $1$ under $p(\mathrm{clone}|X_0)$\footnote{We need this assumption to ensure the density $p(Y_{1:N}|F^M_{1:N})$ is bounded above. Alternatively, one can assume a proper prior on $M$ by picking $\tau$ large but finite in the proof of Prop.~\ref{prop: marginal lik proof}.}.
\begin{proposition}\label{prop: post converge proof}
    \textbf{(Proof of Prop.~\ref{prop: post converge})}
    Assume $\mathrm{clone}\mapsto p(X|\mathrm{clone})$ is measurable and injective, $p(X_0, \hat X_{1:N})\neq 0$, and $\mathrm{Cov}(F_{1:N})>\epsilon$ or some $\epsilon>0$ with probability $1$ under $p(\mathrm{clone}|X_0)$.
    Then, as $M\to\infty$, the approximate and true posteriors converge in total variation --
    \[\Vert \tilde p_M(X_{1:M}|Y_{1:N}, \hat X_{1:N}, X_0)-p(X_{1:M}|Y_{1:N}, \hat X_{1:N}, X_0)\Vert_{\mathrm{TV}}\to 0.\]
\end{proposition}
\begin{proof}
    By our assumptions, by Doob's theorem \citep{Miller2018-wx}, if $\mathrm{clone}\sim p(\mathrm{clone}|X_0)$ and $X_m\sim p(X|\mathrm{clone})$ iid, with probability $1$,
    \[p(\hat X_n|X_{1:M})\to p(\hat X_n|X_{1:M})\]
    for each $n$.
    In particular, we get that $F^M_{1:N}\to F_{1:N}$ and therefore, since, assuming $\mathrm{Cov}(F_{1:N})>\epsilon$, $p(Y_{1:N}|F^M_{1:N})$ is a bounded function of $F_{1:N}$,
    \[E_{X_{1:M}, \mathrm{clone}\sim p(X_{1:M}|\mathrm{clone})p(\mathrm{clone}|X_0)}|p(Y_{1:N}|F^M_{1:N})-p(Y_{1:N}|F_{1:N})|\to 0.\]
    
    Now we show that the the normalizing constants of the approximate and true posteriors converge.
    \begin{equation}
        \begin{aligned}
            Z_M:=\int p(Y_{1:N}|F^M_{1:N}) dp(X_{1:M}|X_0) =&\int p(Y_{1:N}|F^M_{1:N}) \prod_{m=1}^Mdp(X_m|\mathrm{clone})  dp(\mathrm{clone}|X_0)\\
            =&E_{X_{1:M}, \mathrm{clone}\sim p(X_{1:M}|\mathrm{clone})p(\mathrm{clone}|X_0)}p(Y_{1:N}|F^M_{1:N})\\
            \to &E_{\mathrm{clone}\sim p(\mathrm{clone}|X_0)}p(Y_{1:N}|F_{1:N})\\
            =& \int p(Y_{1:N}|F_{1:N}) dp(\mathrm{clone}|X_0)=:Z
        \end{aligned}
    \end{equation}
    By our assumption that $p(X_0, \hat X_{1:N})\neq 0$, for a set of $\mathrm{clone}$ of probability greater than $0$ under $p(\mathrm{clone}|X_0)$, $p(\hat X_n|\mathrm{clone})>0$ for all $n$; for this set, $p(Y_{1:N}|F_{1:N})>0$ and therefore $Z>0$.

    Now we show that the approximate and true posteriors converge in total variation:
    \begin{equation}
        \begin{aligned}
            \sum_{X_{1:M}}&\left\vert\frac{p(Y_{1:N}|F^M_{1:N}) p(X_{1:M}|X_0)}{Z_M} -\frac{\int p(X_{1:M}|\mathrm{clone})p(Y_{1:N}|F_{1:N}) dp(\mathrm{clone}|X_0)}{Z}\right\vert\\
            \leq &\sum_{X_{1:M}}p(Y_{1:N}|F^M_{1:N}) p(X_{1:M}|X_0)\left\vert Z_M^{-1}-Z^{-1}\right\vert\\
            &+\frac 1 Z \sum_{X_{1:M}}\left|\int p(X_{1:M}|\mathrm{clone})\left(p(Y_{1:N}|F^M_{1:N})-p(Y_{1:N}|F_{1:N})\right)dp(\mathrm{clone}|X_0)\right|\\
            \leq &\frac{|Z-Z_M|}{Z}+\frac 1 Z E_{p(\mathrm{clone}|X_0)}|p(Y_{1:N}|F_{1:N})-p(Y_{1:N}|F^M_{1:N})|\to 0.
        \end{aligned}
    \end{equation}
\end{proof}

\end{document}